\documentclass{article}

\PassOptionsToPackage{numbers, compress}{natbib}
    \usepackage[final]{neurips_2022}

\usepackage[utf8]{inputenc} % allow utf-8 input
\usepackage[T1]{fontenc}    % use 8-bit T1 fonts
\usepackage{hyperref}       % hyperlinks
\usepackage{url}            % simple URL typesetting
\usepackage{booktabs}       % professional-quality tables
\usepackage{amsfonts}       % blackboard math symbols
\usepackage{nicefrac}       % compact symbols for 1/2, etc.
\usepackage{microtype}      % microtypography
\usepackage{xcolor}         % colors

\usepackage{amsmath}
\usepackage{amssymb}
\usepackage{mathtools}
\usepackage{amsthm}
\newtheorem{theorem}{Theorem}
\newtheorem{lemma}{Lemma}

\newtheorem{proposition}{Proposition}

\usepackage[capitalize,noabbrev]{cleveref}

\usepackage{bm}
\usepackage{dsfont}
\usepackage{multirow}
\usepackage{array}
\usepackage{bbding}
\usepackage{lipsum}
\usepackage{verbatim}
\usepackage{wrapfig}
\usepackage{subfigure}

\usepackage{pifont}

\newcommand{\xmark}{\text{\ding{55}}}

\def \ind {\mathds{1}}
\def \cD {\mathcal{D}}
\def \cA {\mathcal{A}}

\def \cR {\mathcal{R}}
\def \cL {\mathcal{L}}
\def \cX {\mathcal{X}}

\def \cT {\mathcal{T}}
\def \cS {\mathcal{S}}
\def \cP {\mathcal{P}}
\def \cN {\mathcal{N}}

\def \bbE {\mathbb{E}}
\def \bbA {\mathbb{A}}
\def \bbB {\mathbb{B}}
\def \bbC {\mathbb{C}}
\def \bbD {\mathbb{D}}
\def \bbG {\mathbb{G}}
\def \bbH {\mathbb{H}}
\def \bbI {\mathbb{I}}
\def \bbJ {\mathbb{J}}

\def \vx {\bm{x}}
\def \vp {\bm{p}}

\def \vw {\bm{w}}

\def \vdelta {\bm{\delta}}
\def \vxi {\bm{\xi}}

\def \vmu {\bm{\mu}}
\def \vI {\bm{I}}

\def \hyp {\text{hyp}}
\def \nat {\text{nat}}
\def \adv {\text{adv}}
\def \rob {\text{rob}}
\def \linf {\ell_{\infty}}
\def \sign {\operatorname{sign}}
\def \Pr {\operatorname{Pr}}

\newcommand{\magic}{\vspace{-0.8em}}

\title{Can Adversarial Training Be Manipulated By Non-Robust Features?}

\author{
Lue Tao\textsuperscript{\rm 1}\footnotemark[1]\quad\quad 
Lei Feng\textsuperscript{\rm 2,3}\quad\quad 
Hongxin Wei\textsuperscript{\rm 4}\\
\textbf{Jinfeng Yi}\textsuperscript{\rm 5}\quad\quad 
\textbf{Sheng-Jun Huang}\textsuperscript{\rm 6}\quad\quad 
\textbf{Songcan Chen}\textsuperscript{\rm 6}\footnotemark[2]\\
\normalsize \textsuperscript{\rm 1}National Key Laboratory for Novel Software Technology, Nanjing University, Nanjing, China\\
\normalsize \textsuperscript{\rm 2}Chongqing University, Chongqing, China\\
\normalsize \textsuperscript{\rm 3}RIKEN Center for Advanced Intelligence Project, Japan\\
\normalsize \textsuperscript{\rm 4}Nanyang Technological University, Singapore\\
\normalsize \textsuperscript{\rm 5}JD AI Research, Beijing, China\\
\normalsize \textsuperscript{\rm 6}MIIT Key Laboratory of Pattern Analysis and Machine Intelligence, \\
\normalsize Nanjing University of Aeronautics and Astronautics, Nanjing, China\\
}

\begin{document}

\maketitle

\renewcommand{\thefootnote}{\fnsymbol{footnote}}
\footnotetext[1]{Work done when Lue Tao was a master's student at Nanjing University of Aeronautics and Astronautics.}
\footnotetext[2]{Corresponding author: Songcan Chen <s.chen@nuaa.edu.cn>.}
\renewcommand{\thefootnote}{\arabic{footnote}}

\begin{abstract}
  Adversarial training, originally designed to resist test-time adversarial examples, has shown to be promising in mitigating \textit{training-time availability attacks}. This defense ability, however, is challenged in this paper. We identify a novel threat model named \textit{stability attack}, which aims to hinder \textit{robust} availability by slightly manipulating the training data. Under this threat, we show that adversarial training using a conventional defense budget $\epsilon$ provably fails to provide test robustness in a simple statistical setting, where the non-robust features of the training data can be reinforced by $\epsilon$-bounded perturbation. Further, we analyze the necessity of enlarging the defense budget to counter stability attacks. Finally, comprehensive experiments demonstrate that stability attacks are harmful on benchmark datasets, and thus the adaptive defense is necessary to maintain robustness.\footnote{Our code is available at \url{https://github.com/TLMichael/Hypocritical-Perturbation}.}
\end{abstract}

\section{Introduction}

Robustness to input perturbations is crucial to machine learning deployment in various applications, such as spam filtering~\cite{dalvi2004adversarial} and autonomous driving~\cite{bojarski2016end}. One of the most popular methods for improving test robustness is \textit{adversarial training}~\cite{madry2018towards, athalye2018obfuscated}. By augmenting the training data with $\epsilon$-bounded and on-the-fly crafted adversarial examples, adversarial training helps the learned model resist test-time perturbations~\cite{madry2018towards}.

% \begin{wrapfigure}{r}{0.5\linewidth}
%   \centering
%   \vspace{-30px}
%   \includegraphics[width=\linewidth]{fig/diagram-problem.pdf}
%   \vspace{-18px}
%   \caption{
%     An illustration of stability attacks, where the training data is slightly perturbed to hinder adversarial training.
%   }
%   \label{fig:diagram-problem}
%     \vspace{-15px}
% \end{wrapfigure}

\begin{figure}[t]
\begin{center}
\includegraphics[width=0.5\columnwidth]{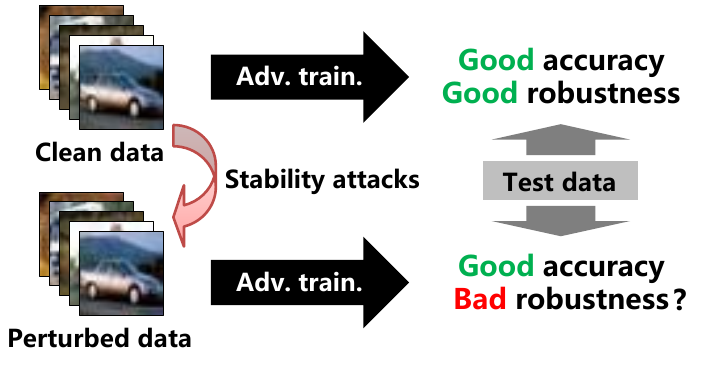}
\caption{
An illustration of stability attacks, where the training data is slightly perturbed to hinder adversarial training.
}
\label{fig:diagram-problem}
\end{center}
\end{figure}

On the other hand, machine learning systems are vulnerable to \textit{training-time availability attacks}~\cite{biggio2018wild}. In particular, small perturbations applied into the training data (before training) suffice to degrade the overall test performance of naturally trained models~\cite{feng2019learning, huang2021unlearnable}. Fortunately, recent work has proven that adversarial training~\cite{madry2018towards} is capable of mitigating this type of threat~\cite{tao2021provable}. In other words, even if the training data is manipulated to maximize the test error, 
considerable accuracy on clean test data can still be achieved by adversarially trained models.
% adversarially trained models can still achieve considerable accuracy on clean test data. 
% The rationale behind this capability is that training models to minimize the adversarial risk on the manipulated data ensures the learning of the clean data, thereby maintaining test accuracy~\cite{tao2021provable}. 
However, previous work hardly inspects the test robustness of the models, which is what adversarial training was originally proposed for~\cite{goodfellow2014explaining, madry2018towards}. 
This naturally raises the following question:

\begin{center}
    \em Are the models adversarially trained on the manipulated data robust to test-time perturbations?\\
    \ \\
\end{center}

In this work, we show that conventional adversarial training may fail to provide test robustness when the training data is manipulated by an adversary, and thus an adaptive defense is necessary to resolve this issue.
Our contributions are summarized as follows:
\begin{enumerate}
    \item We introduce a novel threat model called \textit{stability attack}, where an adversary aims to degrade the overall test robustness of adversarially trained models by slightly perturbing the training~data. \cref{fig:diagram-problem} illustrates the threat of stability attacks.
    \item We show that adversarial training using a conventional defense budget \textit{provably} fails under stability attacks in a simple statistical setting. Specifically, a defense budget of $\epsilon$ will produce models that are \textit{not} robust to $\epsilon$-bounded adversarial examples when the training data is hypocritically perturbed.
    \item We unveil that the aforementioned vulnerability stems from the existence of the non-robust (predictive, yet brittle) features~\cite{ilyas2019adversarial} in the original training data. When the non-robust features are reinforced by hypocritical perturbations, the conventional defense budget will be insufficient to offset the negative impact.
    \item We further show that a defense budget of $2\epsilon$ is capable of resisting any stability attack for adversarial training, while the budget can be reduced to $\epsilon+\eta$ in a simple statistical setting, where $\eta$ is the magnitude of the non-robust features.
    \item We demonstrate that stability attacks are harmful to conventional adversarial training on benchmark datasets. In addition, our empirical study suggests that enlarging the defense budget is essential for mitigating hypocritical perturbations.
\end{enumerate} 
To the best of our knowledge, this is the first work that studies the robustness of adversarial training against stability attacks. Both theoretical and empirical evidences show that the conventional defense budget $\epsilon$ is insufficient under the threat of $\epsilon$-bounded training-time perturbations. Our findings suggest that practitioners should consider a larger defense budget of no more than $2\epsilon$ (practically, about $1.5\epsilon \sim 1.75\epsilon$) to achieve a better $\epsilon$-robustness.

\section{Threat Models}

In this section, we formally introduce the threat model of~stability attacks. We begin by revisiting the concepts of natural risk, adversarial risk, and delusive attacks. These concepts naturally give rise to our formulation of stability attacks.

\subsection{Preliminaries}

\paragraph{Setup.}
% \textbf{Setup.}
We consider a classification task with input-label pairs $(\vx, y)$ from an underlying distribution $\cD$ over $\cX \times [k]$. The goal is to learn a (robust) classifier $f: \cX \rightarrow [k]$ that predicts a label $y$ for a given input $\vx$. 

\paragraph{Natural training (NT).} 
% \textbf{Natural training (NT).} 
Most learning algorithms aim to maximize the generalization performance on unperturbed examples, i.e., natural accuracy. The goal is to minimize the natural risk defined as:
\begin{equation}
\label{eq:nat_risk}
    \cR_{\nat}(f) \coloneqq \underset{(\vx, y) \sim \cD}{\bbE} \left[\cL(f(\vx), y)\right].
\end{equation}

\paragraph{Adversarial training (AT).}
% \textbf{Adversarial training (AT).}
Since the risk of adversarial examples (a.k.a.~evasion attacks) was found to be unexpectedly high~\citep{biggio2013evasion, szegedy2013intriguing}, it has become increasingly important to defend the learner against the worst-case perturbations~\citep{goodfellow2014explaining, madry2018towards}. In this context, the goal is to train a model that has low \textit{adversarial risk} given a defense budget $\epsilon$:
\begin{equation}
\label{eq:adv_risk}
    \cR_{\adv}(f) \coloneqq \underset{(\vx, y) \sim \cD}{\bbE} \left[\max_{\vdelta \in \Delta} \cL(f(\vx+\vdelta), y)\right],
\end{equation}
where we choose $\Delta$ to be the set of $\epsilon$-bounded perturbations, i.e., $\Delta = \{\vdelta \mid \|\vdelta\| \le \epsilon\}$. This choice is the most common one in the context of adversarial examples~\citep{tsipras2018robustness}. To simplify the notation, we refer to the robustness with respect to this set as $\epsilon$-robustness. It is worth noting that $\cR_{\adv}(f) \ge \cR_{\nat}(f)$ always holds for any $f$, and the equation holds when $\epsilon=0$.

\begin{table*}[t]
  \centering
  \caption{
  Comparisons between evasion attacks, delusive attacks, and stability attacks. 
  }
  \label{tab:threat-comparison}
  \vspace{0.5ex}
  % \vspace{1ex}
  \begin{center}
  \begin{small}
  \begin{tabular}{@{}l|cc|cc@{}}
  \toprule
  Threat model &
    \begin{tabular}[c]{@{}c@{}}Training-time \\ perturbation\end{tabular} &
    \begin{tabular}[c]{@{}c@{}}Test-time \\ perturbation\end{tabular} &
    Learning scheme &
    Test performance \\ \midrule
  None                      & $\xmark$                      & $\xmark$                      & NT                & Good \\ \midrule
  \multirow{2}{*}{Evasion attacks~\cite{biggio2013evasion, szegedy2013intriguing, goodfellow2014explaining, madry2018towards}}  & \multirow{2}{*}{$\xmark$}     & \multirow{2}{*}{$\checkmark$} & NT                & Bad  \\
                                    &                               &                               & AT            & Good \\ \midrule
  \multirow{2}{*}{Delusive attacks~\cite{newsome2006paragraph, feng2019learning, huang2021unlearnable, tao2021provable}} & \multirow{2}{*}{$\checkmark$} & \multirow{2}{*}{$\xmark$}     & NT                & Bad  \\
                                    &                               &                               & AT            & Good \\ \midrule
  \multirow{2}{*}{Stability attacks (this paper)} &
    \multirow{2}{*}{$\checkmark$} &
    \multirow{2}{*}{$\checkmark$} &
    AT (conventional) &
    Bad \\
                                    &                               &                               & AT (our improved) & Good \\ \bottomrule
  \end{tabular}
  \end{small}
  \end{center}
  % \vspace{-1.5em}
  \magic
  \end{table*}

\paragraph{Delusive attacks.}
% \textbf{Delusive attacks.}
Delusive attacks, which belong to training-time availability attacks, aim to prevent the learner from producing an accurate model by manipulating the training data ``imperceptibly''~\citep{newsome2006paragraph}. 
Concretely, the features of the training data can be perturbed, while the labels should remain correct~\cite{feng2019learning, nakkiran2019a, shen2019tensorclog, huang2021unlearnable, pmlr-v139-yuan21b, evtimov2021disrupting, tao2021provable, fowl2021adversarial, fu2022robust}. This malicious task can be formalized into the following bi-level optimization problem:
\begin{equation}
\label{eq:delusive_attacks}
\begin{gathered}
    \max_{\cP \in \cS} \underset{(\vx, y) \sim \cD}{\bbE} \left[\cL(f_{\cP}(\vx), y)\right] \\
    \text{s.t. } f_{\cP} \in \underset{f}{\arg\min} \underset{(\vx_i, y_i) \in \cT}{\textstyle \sum} \left[\cL(f(\vx_i + \vp_i), y_i)\right].
\end{gathered}
\end{equation}
Here, the adversary aims to maximize the natural risk of the model $f_{\cP}$ (that is trained on the manipulated training set) by applying the generated perturbations $\cP=\{\vp_i\}_{i=1}^n$ into the original training set $\cT = \{(\vx_i, y_i)\}_{i=1}^n$. The commonly used feasible region is $\cS=\{\{\vp_i\}_{i=1}^n \mid \|\vp_i\| \le \epsilon\}$. 
% Some of the perturbations $\vp_i$ can be set to $\bm{0}$ to restrict their proportion.

Generally, solving~\cref{eq:delusive_attacks} is computationally prohibitive for neural networks~\cite{tao2021provable, fowl2021adversarial}. Thus, various heuristic methods are proposed to achieve the goal. Among them, a representative method is the \textit{hypocritical perturbation}~\cite{tao2020false, tao2021provable}, crafted as follows:
\begin{equation}
\label{eq:hypocritical_perturbation}
  \min_{\|\vp_i\| \le \epsilon} \cL(f_{\text{craft}}(\vx_i + \vp_i), y_i),
\end{equation}
where $f_{\text{craft}}$ is called the \textit{crafting model}, pre-trained before generating poisons. \citet{tao2021provable} simply adopted a naturally trained classifier as the crafting model, while~\citet{huang2021unlearnable} proposed a min-min bi-level optimization process to pre-train the crafting model. \citet{fu2022robust} further built their crafting model via a min-min-max three-level optimization process, and generated their poisons by replacing \cref{eq:hypocritical_perturbation} with a min-max bi-level objective.

Another representative method of delusive attacks is the \textit{adversarial perturbation}, crafted by solving
\begin{equation}
\label{eq:adversarial_perturbation}
\max_{\|\vp_i\| \le \epsilon} \cL(f_{\text{craft}}(\vx_i + \vp_i), y_i).
\end{equation}
\citet{tao2021provable} and \citet{fowl2021adversarial} both found that applying the adversarial perturbation to the training data is very effective at compromising naturally trained models. However, adversarial training has proven to be promising in defending against various delusive attacks~\cite{tao2021provable}.

\subsection{Stability Attacks}
\label{sec:stability-attacks}

% \paragraph{Stability attacks.}
In contrast to delusive attacks that aim at increasing the natural risk, stability attacks attempt to maximize the adversarial risk of the learner by slightly perturbing the training data:
\begin{equation}
\label{eq:stability_attacks}
\begin{gathered}
    \max_{\cP \in \cS} \underset{(\vx, y) \sim \cD}{\bbE} \left[\max_{\vdelta \in \Delta}\cL(f_{\cP}(\vx+\vdelta), y)\right],
\end{gathered}
\end{equation}
where $f_{\cP}$ denotes the victim model, which is naturally or adversarially trained on the perturbed data.
In other words, stability attacks seek to hinder the \textit{robust} availability of the training data.
\cref*{tab:threat-comparison} shows the comparisons among different threat models.

The goal of stability attacks can be immediately achieved for naturally trained models, since they have already incurred high adversarial risk, even if the training data is clean~\cite{szegedy2013intriguing}. To ease the problem of high adversarial risk, adversarial training has been widely used to improve model's adversarial robustness~\citep{gowal2020uncovering, croce2021robustbench}. 
Hence, the main goal of stability attacks becomes to compromise the test robustness of adversarially trained models.

Note that \cref{eq:stability_attacks} is a multi-level optimization problem that is not easy to solve, our next question is how to conduct effective stability attacks against adversarial training.
In the following sections, we introduce an effective stability attack method and analyze the cost of resisting it.

\paragraph{Remark 1.}
% \textbf{Remark 1.}
This work focuses on adding bounded pertubrations as small as possible. We mostly assume that the adversary has full control of training data (instead of changing a few) by following previous works~\cite{feng2019learning, huang2021unlearnable, tao2021provable, fowl2021preventing, fowl2021adversarial, fu2022robust}. This is a realistic assumption~\cite{feng2019learning, fowl2021adversarial}. For instance, in some applications an organization may agree to release some internal data for peer assessment, while preventing competitors from easily building a model with high test robustness; this can be achieved by perturbing the entire dataset via stability attacks before releasing. Moreover, this assumption enables a worst-case analysis of the robustness of adversarial training, which may facilitate important theoretical implications.

\section{How to Manipulate Adversarial Training}
\label{sec:how-to-manipulate}

Previous work suggests that adversarial training could defend against both evasion attacks and delusive attacks~\cite{madry2018towards, tao2021provable}. However, in this paper, we show that adversarial training using a conventional defense budget $\epsilon$ may not be sufficient to provide $\epsilon$-robustness when confronted with stability attacks. In particular, we present a simple theoretical model where the conventional defense scheme provably fails when the training data is hypocritically perturbed.

\paragraph{The binary classification task.} 
% \textbf{The binary classification task.} 
The data model is largely based on the setting proposed by~\citet{tsipras2018robustness}, which draws a distinction between \textit{robust features} and \textit{non-robust features}.
Specifically, it consists of input-label pairs $(\vx, y)$ sampled from a Gaussian mixture distribution $\cD$ as follows:
\begin{equation}
\label{eq:mixGau}
\begin{aligned}
    & y \stackrel{u.a.r}{\sim}\{-1,+1\}, \quad x_1 \sim \cN(y, \sigma^2), \quad x_2, \ldots, x_{d+1} \stackrel{i.i.d}{\sim} \cN(\eta y, \sigma^2),
\end{aligned}
\end{equation}
where $\eta$ is much smaller than $1$ (i.e., $0 < \eta \ll 1$). Hence, samples from $\cD$ consist of a robust feature ($x_1$) that is \textit{strongly} correlated with the label, and $d$ non-robust features ($x_2, \ldots, x_{d+1}$) that are \textit{very weakly} correlated with it. Typically, an adversary can manipulate a large number of non-robust features - e.g. $d = \Theta(1/\eta^2)$ will suffice.

Before introducing the way to hinder robust availability, we briefly illustrate the success of adversarial training when the training data is unperturbed.

\paragraph{Natural and robust classifiers.}
% \textbf{Natural and robust classifiers.}
For standard classification, we consider a natural classifier:
\begin{equation}
\label{eq:std_classifier}
    f_{\nat}(\vx) \coloneqq \sign(\vw_{\nat}^{\top} \vx), \,\, \text{where } \vw_{\nat}\coloneqq[1, \eta, \ldots, \eta],
\end{equation}
which is a minimizer of the natural risk~(\ref{eq:nat_risk}) with 0-1 loss on the data~(\ref{eq:mixGau}), i.e., the Bayes optimal classifier. However, in the adversarial setting, this natural classifier is quite brittle.
Thus, it is imperative to obtain a robust classifier:
\begin{equation}
\label{eq:rob_classifier}
    f_{\rob}(\vx) \coloneqq \sign(\vw_{\rob}^{\top} \vx), \,\, \text{where } \vw_{\rob}\coloneqq[1, 0, \ldots, 0],
\end{equation}
which relies only on the robust feature $x_1$.

\paragraph{Illustration of adversarial accuracy.}
% \textbf{Illustration of adversarial accuracy.}
In the adversarial setting, an adversary that is only allowed to perturb each feature by a moderate $\epsilon$ can effectively subvert the natural classifier~\citet{tsipras2018robustness}. In particular, if $\epsilon=2\eta$, an adversary can essentially force each non-robust feature to be \textit{anti}-correlated with the correct label. The following proposition, proved in \cref{app:proof-thm-adv_accuracy}, gives the adversarial accuracies of the natural classifier $f_{\nat}$~(\ref{eq:std_classifier}) and the robust classifier $f_{\rob}$~(\ref{eq:rob_classifier}).

\begin{proposition}[]
\label{thm:adv_accuracy}
Let $\epsilon=2\eta$ and denote by $\cA_{\textup{adv}}(f)$ the adversarial accuracy, i.e., the probability of a classifier correctly predicting $y$ on the data~(\ref{eq:mixGau}) under $\linf$ perturbations. Then, we have
\begin{equation*}
\begin{aligned}
    \cA_{\textup{adv}}(f_{\textup{nat}}) \le \Pr \left\{ \cN(0, 1) < \frac{1-d\eta^2}{\sigma \sqrt{1+d\eta^2}} \right\}, \quad \cA_{\textup{adv}}(f_{\textup{rob}}) = \Pr \left\{ \cN(0, 1) < \frac{1-2\eta}{\sigma} \right\}.
\end{aligned}
\end{equation*}
\end{proposition}

Proposition~\ref{thm:adv_accuracy} implies that the adversarial accuracy of the natural classifier is $<50\%$ when $d \ge 1/\eta^2$. Even worse, when $\sigma \le 1/3$ and $d \ge 3/\eta^2$, the adversarial accuracy of the natural classifier~(\ref{eq:std_classifier}) is always lower than $1\%$. In contrast, the robust classifier~(\ref{eq:rob_classifier}) yields a much higher adversarial accuracy (always $>50\%$); when $\sigma \le (1-2\eta)/3$, its adversarial accuracy will be higher than $99\%$.

\subsection{Hypocritical Features Are Harmful}
\label{sec:hyp-harmful}

The results above reveal the advantages of robust classifiers over natural classifiers. Note that such a robust classifier can be obtained by adversarial training on the original data~(\ref{eq:mixGau}). However, this defense effect may not hold when the adversary is allowed to perturb the training data. 

We show this by analyzing two representative perturbations: the \textit{adversarial perturbation}~\cite{tao2021provable, fowl2021adversarial} and the \textit{hypocritical perturbation}~\cite{tao2020false, tao2021provable}. When applied into the training data, both perturbations are effective as delusive attacks for naturally trained models. In the following, we show that the former is harmless: adversarial training using a defense budget $\epsilon$ on the adversarially perturbed data can still provide test robustness. In contrast, the latter is harmful: we find that the same defense budget can only produce non-robust classifiers when the training data is hypocritically perturbed.

\paragraph{A harmless case.}
% \textbf{A harmless case.}
Consider an adversary who is capable of perturbing the training data by an attack budget $\epsilon$. The adversary may choose to shift each feature towards $-y$. Hence, the learner would see input-label pairs $(\vx, y)$ sampled i.i.d.~from a training distribution $\cT_{\adv}$ as follows:
\begin{equation}
\label{eq:mixGau_adv}
\begin{aligned}
    & y \stackrel{u.a.r}{\sim}\{-1,+1\}, \quad x_1 \sim \cN((1-\epsilon)y, \sigma^2), \quad x_2, \ldots, x_{d+1} \stackrel{i.i.d}{\sim} \cN((\eta-\epsilon) y, \sigma^2),
\end{aligned}
\end{equation}
where each feature of the samples from $\cT_{\adv}$ is adversarially perturbed by a moderate $\epsilon$. While these samples are deviate significantly from the original distribution $\cD$~(\ref{eq:mixGau}), adversarial training on them using a defense budget $\epsilon$ is still able to neutralize the non-robust features. Formally, in \cref{app:proof-thm-adv-harmless} we prove the following theorem.

\begin{theorem}[Adversarial perturbation is harmless]
\label{thm:adv-harmless}
Assume that the adversarial perturbation in the training data $\cT_{\adv}$~(\ref{eq:mixGau_adv}) is moderate such that $\eta/2 \le \epsilon < 1/2$. Then, the optimal linear $\linf$-robust classifier obtained by minimizing the adversarial risk on $\cT_{\adv}$ with a defense budget $\epsilon$ is equivalent to the robust classifier~(\ref{eq:rob_classifier}).
\end{theorem}
This theorem indicates that the adversarial perturbation is harmless: $\epsilon$-robustness can still be obtained by adversarial training on such perturbed training data.

\paragraph{A harmful case.}
% \textbf{A harmful case.}
However, this defense effect can be completely broken by the hypocritical perturbation. That is, the adversary can instead shift each feature towards $y$. Hence, the learner would see input-label pairs $(\vx, y)$ sampled i.i.d.~from a training distribution $\cT_{\hyp}$ as follows\footnote{To see how this relates to the hypocritical perturbation (\ref{eq:hypocritical_perturbation}), let us consider the logistic loss $\cL(f(\vx), y) = \log(1+\exp(-y f(\vx)))$, and use the natural classifier~(\ref{eq:std_classifier}) as the crafting model. Then, the problem~(\ref{eq:hypocritical_perturbation}) has a closed-form solution $\vp^{*} = y\epsilon \cdot \sign(\vw_{\nat}) = [y\epsilon, \ldots, y\epsilon]$. Applying $\vp^*$ to each $\vx$ yields the distribution $\cT_{\hyp}$.}:
\begin{equation}
\label{eq:mixGau_hyp}
\begin{aligned}
    & y \stackrel{u.a.r}{\sim}\{-1,+1\}, \quad x_1 \sim \cN((1+\epsilon)y, \sigma^2), \quad x_2, \ldots, x_{d+1} \stackrel{i.i.d}{\sim} \cN((\eta+\epsilon) y, \sigma^2),
\end{aligned}
\end{equation}
where each feature of the samples from $\cT_{\hyp}$ is reinforced by a magnitude of $\epsilon$. While these samples become more separable, adversarial training on them using the same defense budget will fail to neutralize the hypocritically perturbed features. Consequently, the resulting classifiers will inevitably have low adversarial accuracy. We make this formal in the following theorem proved in \cref{app:proof-thm-hyp-harmful}.

\begin{theorem}[Hypocritical perturbation is harmful]
\label{thm:hyp-harmful}
The optimal linear $\linf$-robust classifier obtained by minimizing the adversarial risk on the perturbed data $\cT_{\hyp}$~(\ref{eq:mixGau_hyp}) with a defense budget $\epsilon$ is equivalent to the natural classifier~(\ref{eq:std_classifier}).
\end{theorem}
This theorem implies that the conventional defense scheme can only produce non-robust classifiers, whose adversarial accuracy is as low as that of the natural classifier~(\ref{eq:std_classifier}). That is saying, if $\epsilon=2\eta$, $\sigma \le 1/3$ and $d \ge 3/\eta^2$, the classifiers cannot get adversarial accuracy better than $1\%$.

\paragraph{Implications.}
% \textbf{Implications.}
As it turns out, the seemingly beneficial features in $\cT_{\hyp}$~(\ref{eq:mixGau_hyp}) are actually hypocritical. Therefore, the adversary is highly motivated to hide such hypocritical features in the training data, intending to cajole an innocent learner into relying on the non-robust features. 
Intriguingly, we notice that the natural classifier~(\ref{eq:std_classifier}) (i.e., the crafting model used to derive the distribution $\cT_{\hyp}$) actually has $\eta$-robustness. This is essentially because the non-robust features in the data~(\ref{eq:mixGau}) can resist small-magnitude perturbations by design. This motivates us to use ``slightly robust'' classifiers as the crafting model in practice. Indeed, our experimental results show that training the crafting model with $0.25\epsilon$-robustness performs the best for conducting stability attacks. This is different from the previous works~\cite{tao2021provable, fowl2021adversarial} that use naturally trained models as the crafting model for poisoning.

\section{The Necessity of Large Defense Budget}
\label{sec:necessity-large-budget}

We have shown that the hypocritical perturbation is harmful to the conventional adversarial training scheme. Fortunately, it is possible to strengthen the defense by using a larger defense budget, while the crux of the matter is how large the budget is needed.

We find that the minimum value of the defense budget for a successful defense depends on the specific data distribution. Let us first consider the hypocritical data in $\cT_{\hyp}$~(\ref{eq:mixGau_hyp}). In this case, we show that a larger defense budget is necessary in the following theorem proved in \cref{app:proof-thm-eps-eta-necessary}.

\begin{theorem}[$\epsilon+\eta$ is necessary]
\label{thm:eps-eta-necessary}
The optimal linear $\linf$-robust classifier obtained by minimizing the adversarial risk on the perturbed data $\cT_{\hyp}$~(\ref{eq:mixGau_hyp}) with a defense budget $\epsilon+\eta$ is equivalent to the robust classifier~(\ref{eq:rob_classifier}). Moreover, any defense budget lower than $\epsilon+\eta$ will yield classifiers that still rely on all the non-robust features.
\end{theorem}
This theorem implies that, in the case of the mixture Gaussian distribution under the threat of $\epsilon$-bounded hypocritical perturbations, the learner needs a slightly larger defense budget $\epsilon+\eta$ to ensure $\epsilon$-robustness.

While it is challenging to analyze the minimum value of the defense budget in the general case, the following theorem provides an upper bound of the budget.

\begin{theorem}[General case]
\label{thm:two-eps-sufficient}
For any data distribution and any adversary with an attack budget $\epsilon$, training models to minimize the adversarial risk with a defense budget $2\epsilon$ on the perturbed data is sufficient to ensure $\epsilon$-robustness.
\end{theorem}

The proof of \cref{thm:two-eps-sufficient} is deferred in \cref{app:proof-thm-two-eps-sufficient}. It implies that
a defense budget twice to the attack budget should be safe enough under the threat of stability attacks.
\cref{thm:eps-eta-necessary} also suggests that the minimum budget might be much smaller than $2\epsilon$, and it depends on the specific attack methods and data distributions. In the following section, we empirically search for an appropriate defense budget on real-world datasets.

\section{Experiments}
\label{sec:experiments}

In this section, we conduct comprehensive experiments to demonstrate the effectiveness of the hypocritical perturbation as stability attacks on popular benchmark datasets and the necessity of an adaptive defense for better robustness.

We conduct stability attacks by applying hypocritical perturbations into the training set. We focus on an $\linf$ adversary with an \textit{attack budget} $\epsilon_{a}=8/255$ by following~\cite{huang2021unlearnable, pmlr-v139-yuan21b, tao2021provable, fowl2021adversarial}. Our crafting model is adversarially trained with a \textit{crafting budget} $\epsilon_c=2/255$ for 10 epochs before generating perturbations. Unless otherwise specified, we use ResNet-18~\cite{he2016deep} as the default architecture for both the crafting model and the learning model. For adversarial training, we mainly follow the settings in previous studies~\cite{zhang2019theoretically, wang2019improving, rice2020overfitting}. By convention, the \textit{defense budget} is equal to the attack budget, i.e., $\epsilon_d=8/255$. More details on experimental settings are provided in~\cref{sec:experimental-settings}.

\begin{table*}[!t]
  \centering
  \caption{Test robustness (\%) of PGD-AT using a defense budget $\epsilon_d=8/255$ on CIFAR-10.}
  \label{tab:bench-attack}
  % \magic
  % \vspace{1ex}
  \begin{center}
  \begin{small}
  \begin{tabular}{@{}lcccccc@{}}
  \toprule
  Attack               & Natural            & FGSM             & PGD-20           & PGD-100          & CW$_\infty$      & AutoAttack       \\ \midrule
  None (clean)         & 82.17 & 56.63 & 50.63 & 50.35 & 49.37 & 46.99 \\
  DeepConfuse~\cite{feng2019learning}          & 81.25 & 54.14 & 48.25 & 48.02 & 47.34 & 44.79 \\
  Unlearnable Examples~\cite{huang2021unlearnable} & 83.67 & 57.51 & 50.74 & 50.31 & 49.81 & 47.25 \\
  NTGA~\cite{pmlr-v139-yuan21b}                 & 82.99 & 55.71 & 49.17 & 48.82 & 47.96 & 45.36 \\
  Adversarial Poisoning~\cite{fowl2021adversarial} &
    \textbf{77.35} &
    53.93 &
    49.95 &
    49.76 &
    48.35 &
    46.13 \\
  Hypocritical Perturbation (ours) &
    88.07 &
    \textbf{47.93} &
    \textbf{37.61} &
    \textbf{36.96} &
    \textbf{38.58} &
    \textbf{35.44} \\ \bottomrule
  \end{tabular}
  \end{small}
  \end{center}
  % \vspace{-1.5em}
  \magic
  \end{table*}
  
  \begin{table*}[!t]
  \centering
  \caption{Test robustness (\%) of PGD-AT using a defense budget $\epsilon_d=8/255$ across different datasets.}
  \label{tab:bench-attack-datasets}
  % \magic
  % \vspace{1ex}
  \begin{center}
  \begin{small}
  \begin{tabular}{@{}llcccccc@{}}
  \toprule
  Dataset & Attack & Natural & FGSM & PGD-20 & PGD-100 & CW$_\infty$ & AutoAttack \\ \midrule
  \multirow{3}{*}{SVHN} & None & 93.95 & 71.83 & 57.15 & 56.02 & 54.93 & 50.50 \\
   & Adv. & \textbf{87.50 } & \textbf{56.12} & 46.71 & 46.32 & 45.70 & 42.48 \\
   & Hyp. & 96.06   & 59.41 & \textbf{38.17 } & \textbf{37.29 } & \textbf{40.54 } & \textbf{35.43 } \\ \midrule
  \multirow{3}{*}{CIFAR-100} & None & 56.15   & 31.50 & 28.38 & 28.28 & 26.53 & 24.30 \\
   & Adv. & \textbf{52.14  } & 28.59 & 26.19 & 26.09 & 24.36 & 22.71 \\
   & Hyp. & 62.22   & \textbf{26.38} & \textbf{21.51} & \textbf{21.13} & \textbf{21.13} & \textbf{18.74} \\ \midrule
  \multirow{3}{*}{Tiny-ImageNet} & None & \textbf{49.34  } & 25.67 & 22.99 & 22.86 & 20.67 & 18.54 \\
   & Adv. & 49.52   & 22.93 & 20.01 & 19.91 & 18.75 & 16.83 \\
   & Hyp. & 55.92   & \textbf{20.21} & \textbf{15.61} & \textbf{15.26} & \textbf{14.99} & \textbf{12.53} \\ \bottomrule
  \end{tabular}
  \end{small}
  \end{center}
  % \vspace{-1.5em}
  \magic
  \end{table*}
  
  \begin{table*}[!t]
  \centering
  \caption{Test robustness (\%) of PGD-AT using a defense budget $\epsilon_d=8/255$ on CIFAR-10 across different architectures. Test robustness is evaluated by PGD-20. Values in parenthesis denote the accuracy on natural test data.}
  \label{tab:bench-attack-architectures}
  % \magic
  \begin{center}
  \begin{small}
  \begin{tabular}{@{}lcccccc@{}}
  \toprule
  Attack & VGG-16 & GoogLeNet & DenseNet-121 & MobileNetV2 & WideResNet-28-10 \\ \midrule
  None   & 47.37 (77.15) & 50.67 (83.03) & 49.92 (80.08) & 48.51 (80.83) & 53.91 (85.81) \\
  Adv.   & 44.70 (73.24) & 47.72 (79.34) & 48.00 (78.17) & 45.90 (74.61) & 51.01 (82.43) \\
  Hyp.   & \textbf{34.34} (87.20) & \textbf{37.03} (87.61) & \textbf{37.58} (88.04) & \textbf{35.58} (87.04) & \textbf{41.07} (89.14) \\ \bottomrule
  \end{tabular}
  \end{small}
  \end{center}
  % \vspace{-1.5em}
  \magic
  \end{table*}

\subsection{Benchmarking (Non-)Robustness}

\paragraph{Attack evaluation.}
% \textbf{Attack evaluation.}
We compare our crafted hypocritical perturbation to existing methods, which were originally proposed as delusive attacks, including DeepConfuse (which builds an adversarial auto-encoder to generate their perturbations)~\cite{feng2019learning}, Unlearnable Examples (which use a min-min bi-level optimization process to pre-train their crafting model)~\cite{huang2021unlearnable}, NTGA (which adopts neural tangent kernels as its crafting model)~\cite{pmlr-v139-yuan21b}, and Adversarial Poisoning (whose crafting model is simply a naturally trained classifier)~\cite{fowl2021adversarial}. 
It is noteworthy that none of these previous works evaluated the test robustness of their poisoned models.

Results using ResNet-18 on CIFAR-10 are summarized in~\cref{tab:bench-attack}. 
``Natural'' denotes the accuracy on natural test data. Various test-time adversarial attacks are used to evaluate test robustness, including FGSM, PGD-20/100, CW$_{\infty}$ ($\linf$ version of CW loss~\cite{carlini2017towards} optimized by PGD-100), and AutoAttack (a reliable evaluation metric via an ensemble of diverse attacks~\cite{croce2020reliable}).
We observe that the hypocritical perturbation widely outperforms previous training-time perturbations in degrading the test robustness of PGD-AT~\cite{madry2018towards}. This demonstrates that stability attacks are indeed harmful to the conventional defense scheme. We note that our method increases the natural accuracy. This is reasonable, since our analysis in~\cref{sec:hyp-harmful} has implied that the hypocritical perturbation can increase model reliance on the non-robust features, which are predictive but brittle~\cite{ilyas2019adversarial}. 

Moreover, we evaluate the hypocritical perturbation on other benchmark datasets including SVHN, CIFAR-100, and Tiny-ImageNet. Both the crafting model and the victim model use the ResNet-18 architecture. Results are summarized in~\cref{tab:bench-attack-datasets}. ``Hyp.''~denotes the hypocritical perturbation generated by our crafting model. As a comparison, we also evaluate the adversarial perturbation generated using the same crafting model (denoted as ``Adv.''). Again, the results show that the hypocritical perturbations are more threatening than the adversarial perturbations to standard adversarial training. This phenomenon is consistent with our analytical results in~\cref{sec:hyp-harmful}. 

Besides, we find that the hypocritical perturbation can transfer well from ResNet-18 to other architectures, successfully degrading the test robustness of a wide variety of popular architectures including VGG-16~\cite{simonyan2014very}, GoogLeNet~\cite{szegedy2015going}, DenseNet-121~\cite{huang2017densely}, MobileNetV2~\cite{sandler2018mobilenetv2}, and WideResNet-28-10~\cite{zagoruyko2016wide}, as shown in~\cref{tab:bench-attack-architectures}. Note that this is a completely black-box setting where the attacker has no knowledge of the victim model's initialization, architecture, learning rate scheduler, etc.

\paragraph{Adaptive defense.}
% \textbf{Adaptive defense.}
To prevent the harm of stability attacks, our analysis in~\cref{sec:necessity-large-budget} suggests that a larger defense budget would be helpful. We find that this is indeed the case on CIFAR-10. As shown in~\cref{tab:bench-defense}, a large defense budget $\epsilon_d=14/255$ for PGD-AT performs significantly better than the conventional defense budget $\epsilon_d=8/255$. We also combine several data augmentations with PGD-AT as defenses by following~\citet{fowl2021adversarial}. The results show that they are beneficial, while their improvements are inferior to PGD-AT with $\epsilon_d=14/255$. In addition, we adopt other adversarial training variants including TRADES~\cite{zhang2019theoretically} and MART~\cite{wang2019improving} to defend against the hypocritical perturbation, and find that they achieve comparable defense effects with large defense budgets.

Finally, we note that the adaptive defense has several limitations: \textit{i)} robust accuracy is improved at the cost of natural accuracy; \textit{ii)} finding an appropriate defense budget is time-consuming for adversarial training; \textit{iii)} adversarial training with large budgets may lead to learning obstacles such as inherent large sample complexity~\cite{schmidt2018adversarially}. We leave the detailed study of these questions as future work.

\begin{figure*}[t]
  \centering
  \subfigure[Crafting budget]{
  \label{fig:training-budget-for-crafting-model}
     \centering
     \includegraphics[width=0.22\textwidth]{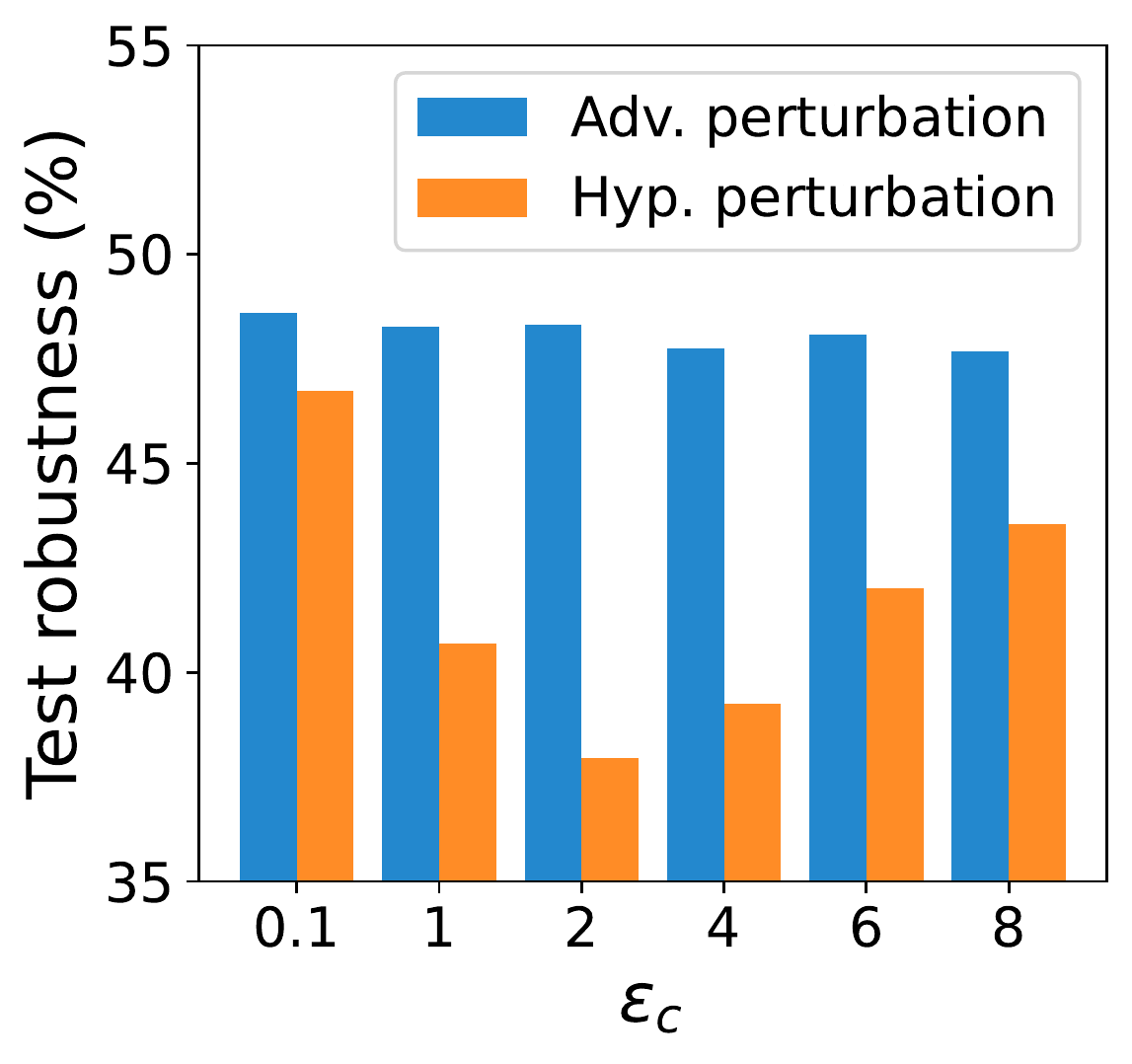}
  }
  \hfill
  \subfigure[Crafting epoch]{
  \label{fig:training-epoch-for-crafting-model}
     \centering
     \includegraphics[width=0.22\textwidth]{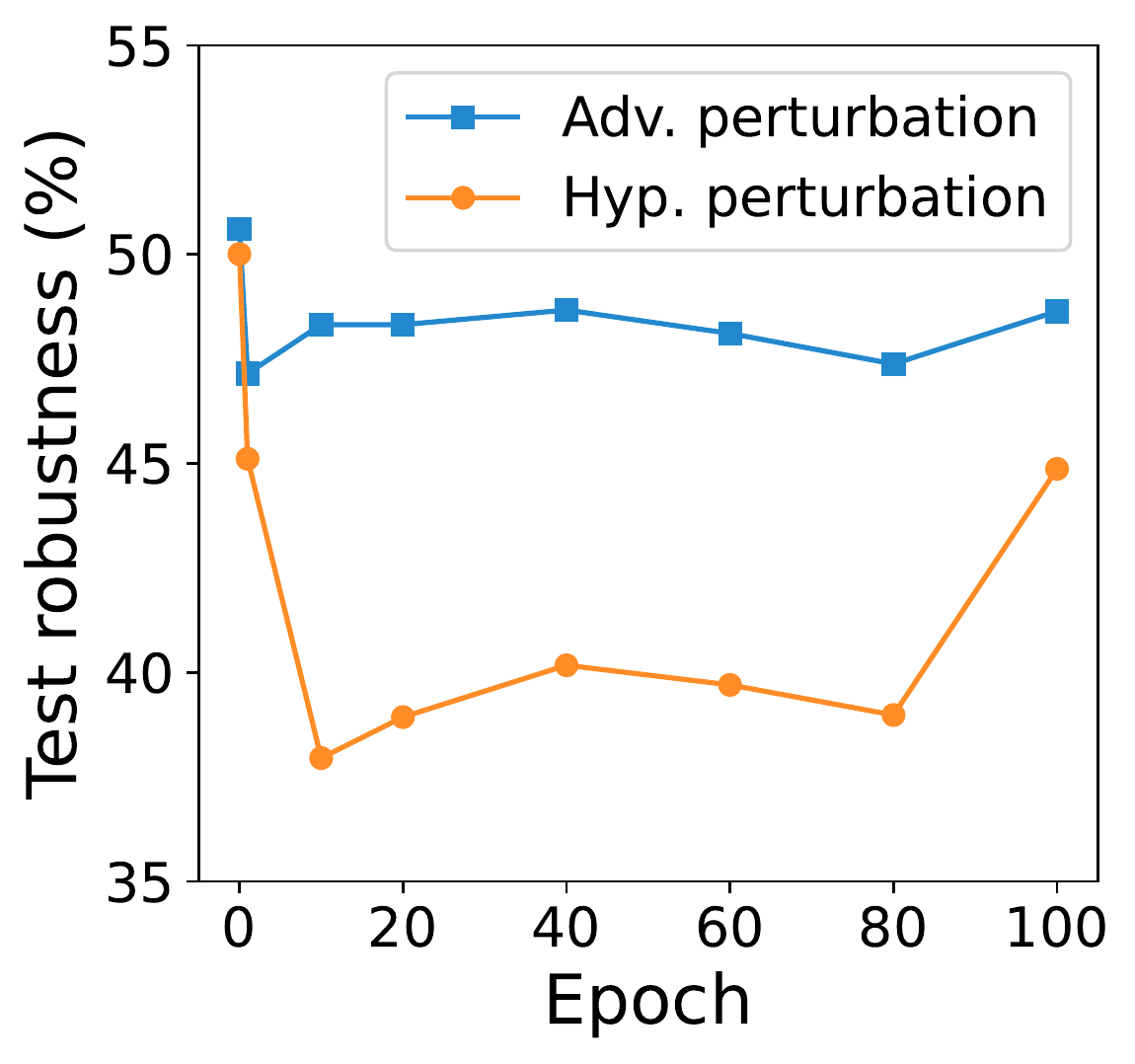}
  }
  \hfill
  \subfigure[Number of PGD steps]{
  \label{fig:pgd-steps-for-crafting-model}
     \centering
     \includegraphics[width=0.22\textwidth]{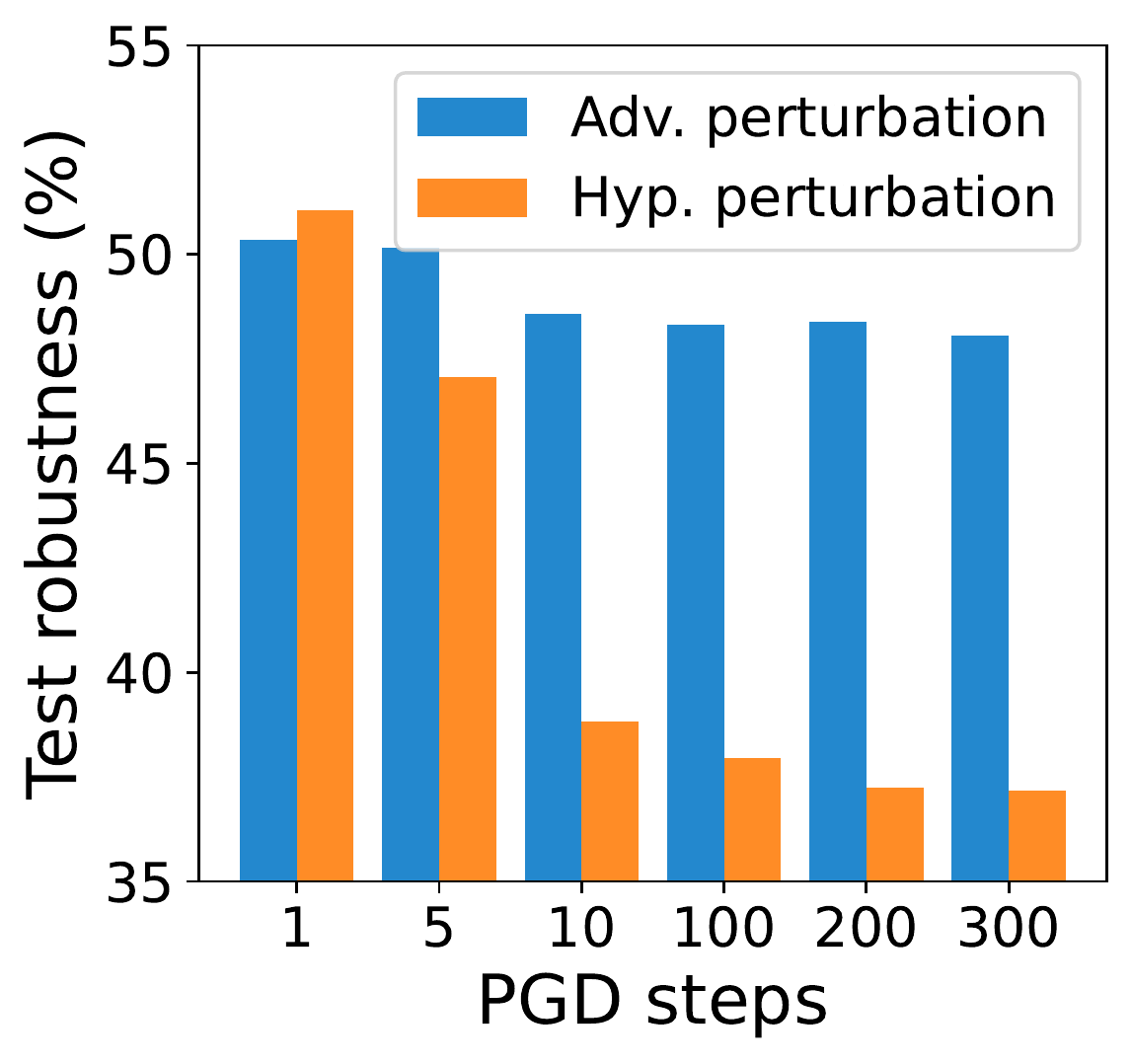}
  }
  \hfill
  \subfigure[Large defense budget]{
  \label{fig:large-defense-budget}
     \centering
     \includegraphics[width=0.22\textwidth]{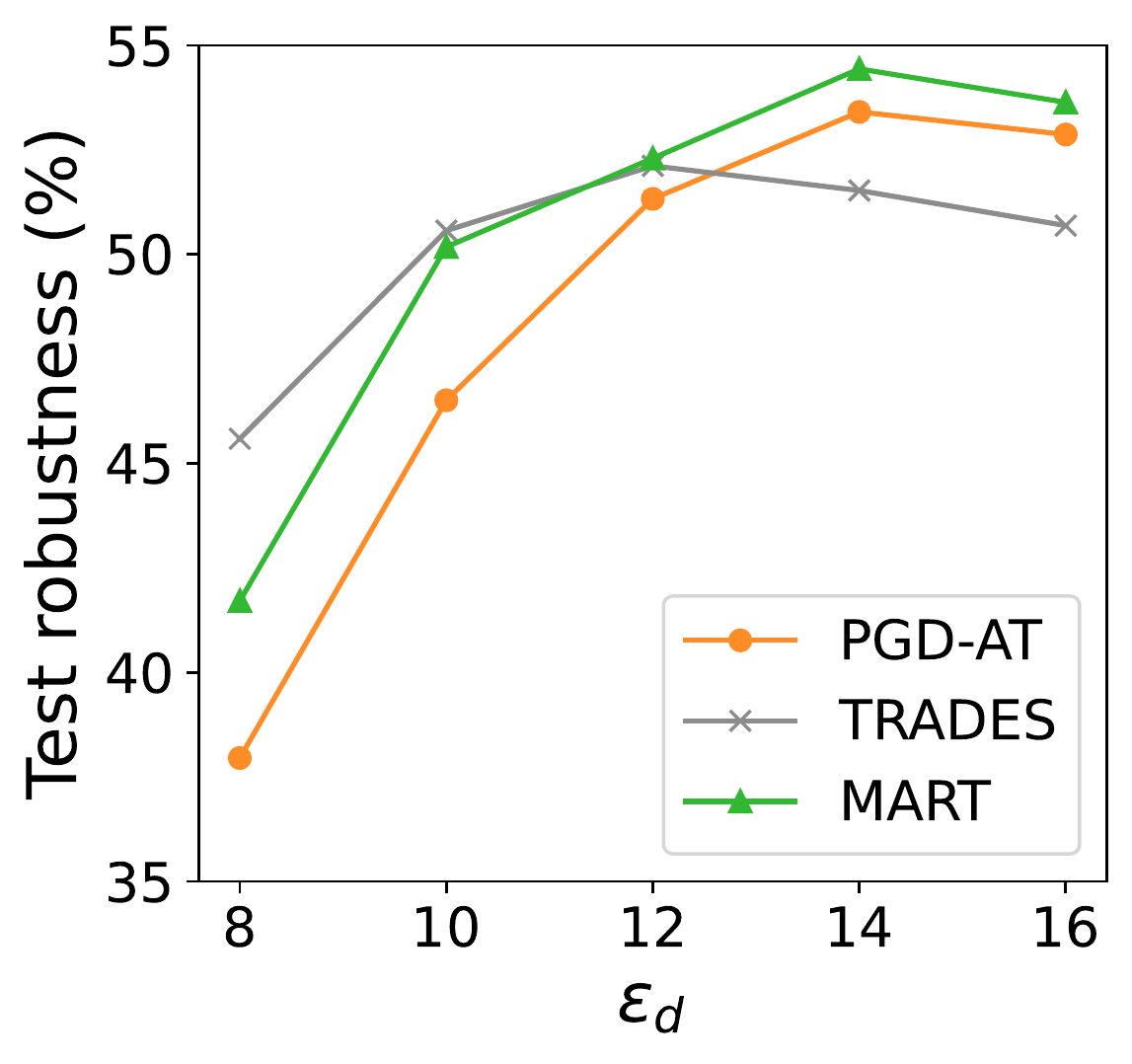}
  }
\vspace{-5px}
% \vspace{-0.5em}
\caption{
The ablation study experiments on CIFAR-10. Test robustness (\%) is evaluated by PGD-20.
}
\label{fig:ablation-study}
% \vspace{-10px}
\magic
\end{figure*}

\begin{table*}[t]
\centering
\caption{Test robustness (\%) of various adaptive defenses on the hypocritically perturbed CIFAR-10.}
\label{tab:bench-defense}
% \magic
\begin{center}
\begin{small}
\begin{tabular}{@{}lcccccc@{}}
\toprule
Defense                & Natural                     & FGSM             & PGD-20           & PGD-100          & CW$_\infty$      & AutoAttack       \\ \midrule
PGD-AT ($\epsilon_d=8/255$)  & 88.07              & 47.93 & 37.61 & 36.96 & 38.58 & 35.44 \\
+ Random Noise         & 87.62            & 47.46 & 38.35 & 37.90 & 39.07 & 36.25 \\
+ Gaussian Smoothing   & 83.95            & 50.96 & 42.80 & 42.34 & 42.41 & 40.07 \\
+ Cutout               & \textbf{88.26  } & 49.23 & 39.77 & 39.25 & 40.38 & 37.61 \\
+ AutoAugment          & 86.24            & 48.87 & 40.19 & 39.65 & 37.66 & 35.07 \\
PGD-AT ($\epsilon_d=14/255$) &
 80.00   &
 56.86 &
 52.92 &
 52.83 &
 \textbf{50.36 } &
 \textbf{48.63 } \\
TRADES ($\epsilon_d=12/255$) & 79.63            & 55.73 & 51.77 & 51.63 & 48.68 & 47.83 \\
MART ($\epsilon_d=14/255$) &
 77.29   &
 \textbf{57.10} &
 \textbf{53.82} &
 \textbf{53.71} &
 49.03 &
 47.67 \\ \bottomrule
\end{tabular}
\end{small}
\end{center}
% \vspace{-1.5em}
\magic
\end{table*}

\begin{table*}[!t]
\centering
\caption{Test robustness (\%) of PGD-AT by adjusting the amount of clean data included in the manipulated CIFAR-10. Test robustness (\%) is evaluated by PGD-20. Values in parenthesis denote the accuracy on natural test data.}
\label{tab:clean-proportion}
\vspace{0.1ex}
% \vspace{1ex}
\begin{center}
\begin{small}
\begin{tabular}{@{}lccccc@{}}
\toprule
Attack$\backslash$Clean proportion & 0.1 & 0.2 & 0.4 & 0.6 & 0.8 \\ \midrule
None (clean subset)    & 30.65 (63.90) & 37.99 (70.99) & 44.95 (77.11) & 47.17 (80.33) & 49.78 (81.60) \\ \midrule
Adversarial Perturbation  & 48.33 (77.71) & 48.23 (76.94) & 49.68 (78.54) & 50.15 (82.46) & 51.21 (82.03) \\
Hypocritical Perturbation & \textbf{41.51} (87.49) & \textbf{43.66} (88.30) & \textbf{46.98} (86.46) & \textbf{49.20} (85.29) & \textbf{50.56} (82.72) \\ \bottomrule
\end{tabular}
\end{small}
\end{center}
% \vspace{-1em}
\magic
\end{table*}

\subsection{Ablation Studies}

In this part, we conduct a set of experiments to provide an empirical understanding of the proposed attack. We train ResNet-18 using PGD-AT on CIFAR-10 by following the same settings described in~\cref{sec:experimental-settings} unless otherwise specified.

\paragraph{Analysis on the crafting method.}
% \textbf{Analysis on the crafting method.}
Different from previous work, we use ``slightly robust'' classifiers as our crafting model. \cref{fig:training-budget-for-crafting-model} shows that this technique greatly improves the potency of the hypocritical perturbation, where the crafting budget $\epsilon_c=2/255$ performs best in degrading test robustness. We also observe that training the crafting model for 10$\sim$80 epochs works well in~\cref{fig:training-epoch-for-crafting-model}, and that optimizing the crafted perturbations over 100 steps performs well in~\cref{fig:pgd-steps-for-crafting-model}.
Finally, we note that \citet{fowl2021adversarial} also tried to use adversarially trained models as the crafting model, but they failed to produce an effective attack in this way. This is mainly because they adopted adversarial perturbations as poisons, which, as we observed, are inferior in degrading test performance.

\paragraph{Ablation on defense budget.}
% \textbf{Ablation on defense budget.}
As discussed in~\cref{sec:necessity-large-budget}, we are motivated to find the appropriate defense budget $\epsilon_d$ in the range $[\epsilon\sim 2\epsilon]$. 
\cref{fig:large-defense-budget} shows that the optimal defense budgets against the hypocritical perturbation are $14/255$, $12/255$, and $14/255$ for PGD-AT, TRADES, and MART, respectively. We also observe that all these adversarial training variants are inferior when using the conventional defense budget $8/255$.

\paragraph{Less data.}
% \textbf{Less data.}
We follow~\citet{fowl2021adversarial} to test the effectiveness of attacks by varying the proportion of clean data and perturbed data. Attacks are then considered effective if they cannot significantly increase performance over training on the clean subset alone. As shown in~\cref{tab:clean-proportion}, the proposed attack often degrades the test robustness below what one would achieve using full clean dataset. More importantly, the hypocritical perturbations are consistently more harmful than the adversarial perturbations. This again verifies the superiority of hypocritical pertubrations as stability attacks.

\begin{wraptable}{r}{0.5\linewidth}
  \centering
  \vspace{-19px}
  \caption{Test robustness (\%) of natural training on CIFAR-10.}
  \label{tab:bench-attack-natural-training}
  \vspace{3px}
  \begin{small}
  \begin{tabular}{@{}lcc@{}}
  \toprule
  Attack                & Natural & PGD-20 \\ \midrule
  None (clean)          & 94.23 & 0.00 \\
  DeepConfuse~\cite{feng2019learning}           & 17.22 & 0.00 \\
  Unlearnable Examples~\cite{huang2021unlearnable}  & 22.72 & 0.00 \\
  NTGA~\cite{pmlr-v139-yuan21b}                  & 11.15 & 0.00 \\
  Adversarial Poisoning~\cite{fowl2021adversarial} & \textbf{8.60} & 0.00 \\
  Hypocritical Perturbation (ours) & 75.92 & 0.00 \\ \bottomrule
  \end{tabular}
  \end{small}
  \vspace{-16px}
\end{wraptable}

\paragraph{Effect on natural training.}
% \textbf{Effect on natural training.}
As a sanity check, we include the test accuracy of naturally trained models on CIFAR-10 in~\cref{tab:bench-attack-natural-training}. It shows that without adversarial training, the test robustness of the models becomes very poor---all models only have $0\%$ accuracy under PGD-20 attack. Thus, the goal of stability attacks is immediately achieved. On the other hand, We find that our method degrades the test accuracy from $94.23\%$ to $75.92\%$, though this is not the main focus of this work. We also observe that Adversarial Poisoning~\cite{fowl2021adversarial} is the most effective method in degrading the test accuracy of naturally trained models. This observation is consistent with~\citet{fowl2021adversarial}.

\section{Related Work}

\paragraph{Adversarial training.}
% \textbf{Adversarial training.}
The presence of non-robust features has been demonstrated on popular benchmark datasets~\cite{ilyas2019adversarial, kim2021distilling}, which naturally leads to model vulnerability to adversarial examples~\cite{tsipras2018robustness, springer2021little}. To improve test robustness against adversarial examples, adversarial training methods have been developed~\cite{goodfellow2014explaining, madry2018towards, wong2018provable, zhang2019theoretically, tramer2019adversarial, pang2020bag, wu2020adversarial, zhang2020geometry, tack2021consistency, wang2021probabilistic}. Usually, adversarial training using a defense budget $\epsilon$ is expected to improve model robustness against $\epsilon$-bounded adversarial examples. Thus, to break this defense, a direct way is to enlarge the typical $\epsilon$-ball used to constrain the attack; however, this may risk changing the true label~\cite{carlini2019critique, tramer2020fundamental}. In this work, we aim to show that it is possible to achieve this by slightly perturbing the training data without enlarging the $\epsilon$-ball.

\paragraph{Data poisoning.}
% \textbf{Data poisoning.}
Data poisoning attacks, which manipulate training data to cause the resulting models to fail during inference~\cite{biggio2018wild}, can be divided into \textit{availability attacks} (to degrade overall test performance)~\cite{biggio2012poisoning, xiao2015feature, munoz2017towards, pang2021accumulative, fowl2021adversarial} and \textit{integrity attacks} (to cause specific misclassifications)~\cite{koh2017understanding, chen2017targeted, shafahi2018poison, zhu2019transferable, geiping2020witches, schwarzschild2021just}.
While the stability attacks studied in this work may be reminiscent of \textit{backdoor attacks}~\cite{chen2017targeted}, we note that they share several key differences.
First, stability attacks aim to hinder adversarial training with well-defined $\epsilon$-robustness, while backdoor attacks mainly focus on embedding malicious behaviors (that can be invoked by pre-specified triggers) into naturally trained models~\cite{goldblum2020dataset, saha2020hidden, turner2019label}.
Second, stability attacks only perturb the inputs slightly, while most works on backdoor attacks require mislabeling~\cite{gu2017badnets, liu2017trojaning, nguyen2020input, li2021invisible, wu2022backdoorbench}. 
Thus, backdoor defenses~\cite{borgnia2021strong, wu2021Adversarial, li2021anti} might not be directly applied to resist stability attacks. 
% \textbf{Concurrent work.}

Additional related works are discussed in Appendix~\ref{app:additional-related-work}.

\section{Conclusion}
\label{sec:conclusion}

In this work, we establish a framework to study the robustness of adversarial training against stability attacks. 
We unveil the threat of stability attacks---small hypocritical perturbations applied into the training data suffice to hinder conventional adversarial training. The conventional defense budget $\epsilon$ is insufficient under the threat. To resist it, we suggest a larger defense budget of no more than $2\epsilon$. Our theoretical analysis explains why hypocritical perturbations are effective as stability attacks---they can mislead the learner by reinforcing the non-robust features. Experiments demonstrate that hypocritical perturbations are harmful to conventional adversarial training on benchmark datasets, and enlarging the defense budget is essential for mitigating hypocritical pertubrations.
Future work includes relaxing the assumption that the adversary perturbs the entire training set and designing more effective stability attacks against adversarial training.

% \begin{ack}
% This work was supported by the National Natural Science Foundation of China (Grant No. 61732006, 62076124, 62076128, 62106028), the National Key R\&D Program of China (2020AAA0107000). Lei Feng was also supported by Chongqing Overseas Chinese Entrepreneurship, Innovation Support Program and CAAI-Huawei MindSpore Open Fund.
% The authors wish to thank the anonymous reviewers for their helpful comments and suggestions.
% \end{ack}

% \clearpage

\small
\bibliographystyle{plainnat}
\bibliography{ref}

\newpage
\appendix
% \begin{center}{\bf {\LARGE Supplementary Material:}}
% \end{center}
\begin{center}{\bf {\Large Supplementary Material: Can Adversarial Training Be Manipulated By Non-Robust Features?}}
\end{center}
\vspace{0.1in}

% \clearpage
% \appendix

\section{Additional Related Work}
\label{app:additional-related-work}

In this part, we discuss several independent (or concurrent) works that are closely related to this work.

\citet{zhu2021understanding} study the effect of conventional adversarial training on differentiating noisy labels, while \citet{zhang2021noilin} show that deliberately injected noisy labels may serve as a regularization that alleviates robust overfitting. Our results focus on the clean-label setting and provide evidence that conventional adversarial training can be hindered without modifying the labels.

\citet{yu2021indiscriminate} suggest explaining the success of availability attacks from the perspective of shortcuts. They further adopt pre-trained models to extract useful features for mitigating model reliance on the shortcuts. This direction is orthogonal to ours.

\citet{liu2021going} improve the effectiveness of unlearnable examples~\cite{huang2021unlearnable} by generating grayscale perturbations and using data augmentations. They also conclude that conventional adversarial training will prevent a drop in accuracy measured both on clean images and adversarial images. Contrary to them, we show that, both theoretically and empirically, conventional adversarial training can be hindered by hypocritical perturbations, and we further analyze the necessity of enlarging the defense budget to resist stability attacks.

\citet{gao2022does} revisit the trade-off between adversarial robustness and backdoor robustness~\cite{weng2020trade}. 
% They conclude that adversarial training is able to achieve backdoor robustness as long as the defense budget surpasses the trigger magnitude. 
They conclude that backdoor attacks are ineffective when the defense budget of adversarial training surpasses the trigger magnitude.
In contrast, our results indicate that stability attacks are still harmful to adversarial training when the defense budget is not large enough. In a simple statistical setting, a defense budget $\epsilon + \eta$ is necessary (where $\eta$ is a positive number). In the general case, a defense budget of $2\epsilon$ is sufficient. In our experiments, a defense budget of about $1.5\epsilon \sim 1.75\epsilon$ provides the best empirical $\epsilon$-robustness.

\citet{wang2022fooling} argue that it is necessary to use robust features for compromising adversarial training. To this end, they adopt a relatively large attack budget $\epsilon_a=32/255$ for crafting their poisons (they use one type of adversarial perturbations), and show that their poisons can decrease the performance of the models trained using smaller defense budgets (such as $\epsilon_d=8/255$ and $\epsilon_d=16/255$). In contrast, we focus on a more realistic setting that does not require a larger attack budget. We demonstrate that it is possible to hinder adversarial training when $\epsilon_a=\epsilon_d$. Furthermore, we provide both theoretical and empirical results showing how to adapt the defense to maintain robustness.

\citet{fu2022robust} explore how to protect data privacy against adversarial training. The main purpose of their poisons is to compromise adversarial training by requiring the perturbation budget of their poisons to be larger than that of adversarial training. In this way, they show that the natural accuracy of the adversarially trained models can be largely decreased, let alone robust accuracy. From this perspective, our work is complementary to theirs. We pursue to not increase the attack budget of stability attacks, keeping it as small as possible. We successfully demonstrate that stability attacks are still harmful to conventional adversarial training without enlarging the attack budget. This makes the threat of stability attacks more insidious than that of~\citet{fu2022robust}.

On the other hand, we find that our implementation of stability attacks using hypocritical perturbations has some similarities to the robust unlearnable examples in~\citet{fu2022robust}. Specifically, although the robust unlearnable examples are generated via a complicated min-min-max optimization process~\cite{fu2022robust}, we notice that their noise generator can be viewed as an adversarially trained model. This implies that the robust error-minimizing (REM) noise~\cite{fu2022robust} might be useful in demonstrating the feasibility of stability attacks. To verify this, we run the source code from the authors with default hyperparameters, and compare our crafted hypocritical perturbation with their generated noise under the setting of stability attacks. For a fair comparison, here we apply a very simple trick called EOT~\cite{athalye2018synthesizing} in our method, since the trick is also used by REM~\cite{fu2022robust}. The additional time cost of the EOT trick is very small and negligible. 

Our experimental results, shown in~\cref{tab:bench-attack-REM}, demonstrate that the robust error-minimizing noise is also effective as stability attacks, though it was originally proposed as a delusive attack. It is noteworthy that the robust accuracy is not evaluated in~\cite{fu2022robust}. In this sense, the effectiveness of REM as an stability attack can be regarded as one of our novel findings. 
Importantly, our method outperforms REM in terms of the robust accuracy against AutoAttack. Since AutoAttack is the most reliable evaluation metric of model robustness among the test-time attacks~\cite{croce2020reliable}, this indicates that our method is reliably more effective than REM in degrading model robustness. 
It is also noteworthy that our method is significantly more efficient than REM, as shown in~\cref{tab:time-cost-comparison}. We note that the efficiency of our method is largely due to the fact that our crafting model is fast to train. Specifically, the time cost of training our crafting model is only 0.3 hours, while it takes 20.8 hours for REM. That is, our crafting model is nearly 70 times faster to train than that of REM. In short, our method is not only more effective, but also more efficient, than REM as a stability attack.

\begin{table*}[!h]
  \centering
  \caption{Comparison with REM~\cite{fu2022robust}: Test robustness (\%) of PGD-AT using a defense budget $\epsilon_d=8/255$ on CIFAR-10. We report mean and standard deviation over 3 random runs.}
  \label{tab:bench-attack-REM}
  % \magic
  % \vspace{1ex}
  \begin{center}
  \begin{scriptsize}
  \begin{tabular}{@{}lcccccc@{}}
  \toprule
  Attack               & Natural            & FGSM             & PGD-20           & PGD-100          & CW$_\infty$      & AutoAttack       \\ \midrule
  None (clean)         & 82.17   $\pm$ 0.71 & 56.63 $\pm$ 0.54 & 50.63 $\pm$ 0.56 & 50.35 $\pm$ 0.59 & 49.37 $\pm$ 0.57 & 46.99 $\pm$ 0.62 \\
  DeepConfuse~\cite{feng2019learning}          & 81.25   $\pm$ 1.52 & 54.14 $\pm$ 0.63 & 48.25 $\pm$ 0.40 & 48.02 $\pm$ 0.40 & 47.34 $\pm$ 0.05 & 44.79 $\pm$ 0.36 \\
  Unlearnable Examples~\cite{huang2021unlearnable} & 83.67   $\pm$ 0.86 & 57.51 $\pm$ 0.31 & 50.74 $\pm$ 0.37 & 50.31 $\pm$ 0.38 & 49.81 $\pm$ 0.24 & 47.25 $\pm$ 0.32 \\
  NTGA~\cite{pmlr-v139-yuan21b}                 & 82.99   $\pm$ 0.40 & 55.71 $\pm$ 0.36 & 49.17 $\pm$ 0.27 & 48.82 $\pm$ 0.30 & 47.96 $\pm$ 0.16 & 45.36 $\pm$ 0.32 \\
  Adversarial Poisoning~\cite{fowl2021adversarial} &
    \textbf{77.35   $\pm$ 0.43} &
    53.93 $\pm$ 0.02 &
    49.95 $\pm$ 0.11 &
    49.76 $\pm$ 0.08 &
    48.35 $\pm$ 0.04 &
    46.13 $\pm$ 0.18 \\
  REM~\cite{fu2022robust} &
    85.63   $\pm$ 1.05 &
    \textbf{42.86 $\pm$ 1.09} &
    {35.40 $\pm$ 0.04} &
    {35.11 $\pm$ 0.09} &
    \textbf{35.24 $\pm$ 0.33} &
    {33.09 $\pm$ 0.24} \\
    % Hypocritical Perturbation (ours) &
    %   88.07   $\pm$ 1.10 &
    %   {47.93 $\pm$ 1.88} &
    %   {37.61 $\pm$ 0.77} &
    %   {36.96 $\pm$ 0.61} &
    %   {38.58 $\pm$ 1.15} &
    %   {35.44 $\pm$ 0.77} \\
      Hypocritical Perturbation (ours) &
        87.60   $\pm$ 0.45 &
        {45.00 $\pm$ 0.77} &
        \textbf{34.89 $\pm$ 0.36} &
        \textbf{34.27 $\pm$ 0.36} &
        {36.28 $\pm$ 0.38} &
        \textbf{32.79 $\pm$ 0.37} \\ \bottomrule
  \end{tabular}
  \end{scriptsize}
  \end{center}
  % \vspace{-1.5em}
  \magic
  \end{table*}
  
  \begin{table*}[!h]
  \centering
  \caption{Comparison with REM~\cite{fu2022robust}: Time cost (min) of poisoning CIFAR-10.}
  \label{tab:time-cost-comparison}
  % \magic
  \begin{center}
  \begin{small}
  \begin{tabular}{@{}lrrr@{}}
  \toprule
  Method & Training the crafting model & Perturbation generation & Total \\ \midrule
  REM~\cite{fu2022robust}   & 1252.4 & 98.1 & 1350.5  \\
  Hypocritical Perturbation (ours)   & \textbf{18.5} & \textbf{17.3} & \textbf{35.8}  \\ \bottomrule
  \end{tabular}
  \end{small}
  \end{center}
  % \vspace{-1.5em}
  \magic
  \end{table*}

\section{Omitted Tables}
\label{appendix.figures}

\begin{table*}[!h]
  \centering
  \caption{Full table of~\cref{tab:bench-attack}: Test robustness (\%) of PGD-AT using a defense budget $\epsilon_d=8/255$ on CIFAR-10. We report mean and standard deviation over 3 random runs.}
  \label{tab:bench-attack-std}
  % \magic
  % \vspace{1ex}
  \begin{center}
  \begin{scriptsize}
  \begin{tabular}{@{}lcccccc@{}}
  \toprule
  Attack               & Natural            & FGSM             & PGD-20           & PGD-100          & CW$_\infty$      & AutoAttack       \\ \midrule
  None (clean)         & 82.17   $\pm$ 0.71 & 56.63 $\pm$ 0.54 & 50.63 $\pm$ 0.56 & 50.35 $\pm$ 0.59 & 49.37 $\pm$ 0.57 & 46.99 $\pm$ 0.62 \\
  DeepConfuse~\cite{feng2019learning}          & 81.25   $\pm$ 1.52 & 54.14 $\pm$ 0.63 & 48.25 $\pm$ 0.40 & 48.02 $\pm$ 0.40 & 47.34 $\pm$ 0.05 & 44.79 $\pm$ 0.36 \\
  Unlearnable Examples~\cite{huang2021unlearnable} & 83.67   $\pm$ 0.86 & 57.51 $\pm$ 0.31 & 50.74 $\pm$ 0.37 & 50.31 $\pm$ 0.38 & 49.81 $\pm$ 0.24 & 47.25 $\pm$ 0.32 \\
  NTGA~\cite{pmlr-v139-yuan21b}                 & 82.99   $\pm$ 0.40 & 55.71 $\pm$ 0.36 & 49.17 $\pm$ 0.27 & 48.82 $\pm$ 0.30 & 47.96 $\pm$ 0.16 & 45.36 $\pm$ 0.32 \\
  Adversarial Poisoning~\cite{fowl2021adversarial} &
    \textbf{77.35   $\pm$ 0.43} &
    53.93 $\pm$ 0.02 &
    49.95 $\pm$ 0.11 &
    49.76 $\pm$ 0.08 &
    48.35 $\pm$ 0.04 &
    46.13 $\pm$ 0.18 \\
  Hypocritical Perturbation (ours) &
    88.07   $\pm$ 1.10 &
    \textbf{47.93 $\pm$ 1.88} &
    \textbf{37.61 $\pm$ 0.77} &
    \textbf{36.96 $\pm$ 0.61} &
    \textbf{38.58 $\pm$ 1.15} &
    \textbf{35.44 $\pm$ 0.77} \\ \bottomrule
  \end{tabular}
  \end{scriptsize}
  \end{center}
  % \vspace{-1.5em}
  \magic
  \end{table*}
  
  \begin{table*}[!h]
  \centering
  \caption{Full table of~\cref{tab:bench-attack-datasets}: Test robustness (\%) of PGD-AT using a defense budget $\epsilon_d=8/255$ across different datasets. We report mean and standard deviation over 3 random runs.}
  \label{tab:bench-attack-datasets-std}
  % \magic
  % \vspace{1ex}
  \begin{center}
  \begin{scriptsize}
  \begin{tabular}{@{}llcccccc@{}}
  \toprule
  Dataset & Attack & Natural & FGSM & PGD-20 & PGD-100 & CW$_\infty$ & AutoAttack \\ \midrule
  \multirow{3}{*}{SVHN} & None & 93.95   $\pm$ 0.21 & 71.83 $\pm$ 1.10 & 57.15 $\pm$ 0.31 & 56.02 $\pm$ 0.33 & 54.93 $\pm$ 0.19 & 50.50 $\pm$ 0.44 \\
   & Adv. & \textbf{87.50   $\pm$ 0.30} & \textbf{56.12 $\pm$ 0.33} & 46.71 $\pm$ 0.25 & 46.32 $\pm$ 0.26 & 45.70 $\pm$ 0.27 & 42.48 $\pm$ 0.21 \\
   & Hyp. & 96.06   $\pm$ 0.01 & 59.41 $\pm$ 0.07 & \textbf{38.17 $\pm$ 0.19} & \textbf{37.29 $\pm$ 0.21} & \textbf{40.54 $\pm$ 0.27} & \textbf{35.43 $\pm$ 0.29} \\ \midrule
  \multirow{3}{*}{CIFAR-100} & None & 56.15   $\pm$ 0.17 & 31.50 $\pm$ 0.16 & 28.38 $\pm$ 0.39 & 28.28 $\pm$ 0.40 & 26.53 $\pm$ 0.27 & 24.30 $\pm$ 0.31 \\
   & Adv. & \textbf{52.14   $\pm$ 0.34} & 28.59 $\pm$ 0.12 & 26.19 $\pm$ 0.11 & 26.09 $\pm$ 0.12 & 24.36 $\pm$ 0.09 & 22.71 $\pm$ 0.11 \\
   & Hyp. & 62.22   $\pm$ 0.11 & \textbf{26.38 $\pm$ 0.11} & \textbf{21.51 $\pm$ 0.06} & \textbf{21.13 $\pm$ 0.02} & \textbf{21.13 $\pm$ 0.23} & \textbf{18.74 $\pm$ 0.10} \\ \midrule
  \multirow{3}{*}{Tiny-ImageNet} & None & \textbf{49.34   $\pm$ 2.61} & 25.67 $\pm$ 0.92 & 22.99 $\pm$ 0.37 & 22.86 $\pm$ 0.36 & 20.67 $\pm$ 0.69 & 18.54 $\pm$ 0.61 \\
   & Adv. & 49.52   $\pm$ 0.19 & 22.93 $\pm$ 0.38 & 20.01 $\pm$ 0.24 & 19.91 $\pm$ 0.24 & 18.75 $\pm$ 0.19 & 16.83 $\pm$ 0.25 \\
   & Hyp. & 55.92   $\pm$ 1.95 & \textbf{20.21 $\pm$ 0.84} & \textbf{15.61 $\pm$ 0.31} & \textbf{15.26 $\pm$ 0.26} & \textbf{14.99 $\pm$ 0.73} & \textbf{12.53 $\pm$ 0.57} \\ \bottomrule
  \end{tabular}
  \end{scriptsize}
  \end{center}
  % \vspace{-1.5em}
  \magic
  \end{table*}

  \begin{table*}[!h]
  \centering
  \caption{Full table of~\cref{tab:bench-defense}: Test robustness (\%) of various adaptive defenses on the hypocritically perturbed CIFAR-10. We report mean and standard deviation over 3 random runs.}
  \label{tab:bench-defense-std}
  % \magic
  \begin{center}
  \begin{scriptsize}
  \begin{tabular}{@{}lcccccc@{}}
  \toprule
  Defense                & Natural                     & FGSM             & PGD-20           & PGD-100          & CW$_\infty$      & AutoAttack       \\ \midrule
  PGD-AT ($\epsilon_d=8/255$)  & 88.07   $\pm$ 1.10          & 47.93 $\pm$ 1.88 & 37.61 $\pm$ 0.77 & 36.96 $\pm$ 0.61 & 38.58 $\pm$ 1.15 & 35.44 $\pm$ 0.77 \\
  + Random Noise         & 87.62   $\pm$ 0.07          & 47.46 $\pm$ 0.08 & 38.35 $\pm$ 0.08 & 37.90 $\pm$ 0.07 & 39.07 $\pm$ 0.20 & 36.25 $\pm$ 0.14 \\
  + Gaussian Smoothing   & 83.95   $\pm$ 0.27          & 50.96 $\pm$ 0.24 & 42.80 $\pm$ 0.40 & 42.34 $\pm$ 0.38 & 42.41 $\pm$ 0.19 & 40.07 $\pm$ 0.29 \\
  + Cutout               & \textbf{88.26   $\pm$ 0.15} & 49.23 $\pm$ 0.42 & 39.77 $\pm$ 0.26 & 39.25 $\pm$ 0.25 & 40.38 $\pm$ 0.25 & 37.61 $\pm$ 0.35 \\
  + AutoAugment          & 86.24   $\pm$ 1.14          & 48.87 $\pm$ 1.01 & 40.19 $\pm$ 0.67 & 39.65 $\pm$ 0.72 & 37.66 $\pm$ 0.88 & 35.07 $\pm$ 0.88 \\
  PGD-AT ($\epsilon_d=14/255$) &
   80.00   $\pm$ 1.91 &
   56.86 $\pm$ 1.42 &
   52.92 $\pm$ 0.86 &
   52.83 $\pm$ 0.86 &
   \textbf{50.36 $\pm$ 1.11} &
   \textbf{48.63 $\pm$ 0.93} \\
  TRADES ($\epsilon_d=12/255$) & 79.63   $\pm$ 0.06          & 55.73 $\pm$ 0.04 & 51.77 $\pm$ 0.15 & 51.63 $\pm$ 0.15 & 48.68 $\pm$ 0.06 & 47.83 $\pm$ 0.02 \\
  MART ($\epsilon_d=14/255$) &
   77.29   $\pm$ 0.87 &
   \textbf{57.10 $\pm$ 0.57} &
   \textbf{53.82 $\pm$ 0.36} &
   \textbf{53.71 $\pm$ 0.34} &
   49.03 $\pm$ 0.47 &
   47.67 $\pm$ 0.51 \\ \bottomrule
  \end{tabular}
  \end{scriptsize}
  \end{center}
  % \vspace{-1.5em}
  % \magic
  \end{table*}

% \begin{table}[!h]
%   \vspace{-0.3em}
%   \centering
%   \caption{Test accuracy (\%) of natural training on CIFAR-10. We report mean and standard deviation over 3 random runs.}
%   \label{tab:bench-attack-natural-training}
%   \magic
%   \begin{center}
%   \begin{scriptsize}
%   \begin{tabular}{@{}lcc@{}}
%   \toprule
%   Attack                & Natural                    & PGD-20          \\ \midrule
%   None (clean)          & 94.23   $\pm$ 0.14         & 0.00 $\pm$ 0.00 \\
%   DeepConfuse           & 17.22   $\pm$ 0.64         & 0.00 $\pm$ 0.00 \\
%   Unlearnable Examples  & 22.72   $\pm$ 0.51         & 0.00 $\pm$ 0.00 \\
%   NTGA                  & 11.15   $\pm$ 0.27         & 0.00 $\pm$ 0.00 \\
%   Adversarial Poisoning & \textbf{8.60 $\pm$ 1.39}   & 0.00 $\pm$ 0.00 \\
%   Hypocritical Perturbation & 75.92   $\pm$ 1.04         & 0.00 $\pm$ 0.00 \\ \bottomrule
%   \end{tabular}
%   \end{scriptsize}
%   \end{center}
%   % \vspace{-1.7em}
%   \magic
%   \end{table}

% \clearpage
\section{Proofs}

In this section, we provide the proofs of our theoretical results in~\cref{sec:how-to-manipulate} and~\cref{sec:necessity-large-budget}.

\subsection{Proof of~Proposition~\ref{thm:adv_accuracy}}
\label{app:proof-thm-adv_accuracy}

\textbf{Proposition~\ref{thm:adv_accuracy} (restated).}
\emph{
  Let $\epsilon=2\eta$ and denote by $\cA_{\textup{adv}}(f)$ the adversarial accuracy, i.e., the probability of a classifier correctly predicting $y$ on the data~(\ref{eq:mixGau}) under $\linf$ perturbations. Then, we have
  \begin{equation*}
  \begin{aligned}
      \cA_{\textup{adv}}(f_{\textup{nat}}) \le \Pr \left\{ \cN(0, 1) < \frac{1-d\eta^2}{\sigma \sqrt{1+d\eta^2}} \right\}, \quad \cA_{\textup{adv}}(f_{\textup{rob}}) = \Pr \left\{ \cN(0, 1) < \frac{1-2\eta}{\sigma} \right\}.
  \end{aligned}
  \end{equation*}
}

\begin{proof}
Recalling that in~\cref{eq:std_classifier}, we have the natural classifier:
\begin{equation}
    f_{\nat}(\vx) \coloneqq \sign(\vw_{\nat}^{\top} \vx), \,\, \text{where } \vw_{\nat}\coloneqq[1, \eta, \ldots, \eta],
\end{equation}
and in~\cref{eq:rob_classifier}, the robust classifier is defined as:
\begin{equation}
    f_{\rob}(\vx) \coloneqq \sign(\vw_{\rob}^{\top} \vx), \,\, \text{where } \vw_{\rob}\coloneqq[1, 0, \ldots, 0].
\end{equation}

Then, the adversarial accuracy of the natural classifier on the data $\cD$~(\ref{eq:mixGau}) is
\begin{equation}
\begin{aligned}
    \cA_{\textup{adv}}(f_{\textup{nat}}) 
    &= 1 - \underset{(\vx, y) \sim \cD}{\Pr} \left\{\exists \|\vdelta\|_{\infty} \le \epsilon, f_{\nat}(\vx+\vdelta) \neq y \right\} \\
    &= 1 - \underset{(\vx, y) \sim \cD}{\Pr} \left\{\min_{\|\vdelta\|_{\infty} \le \epsilon} \left[y \cdot f_{\nat}(\vx+\vdelta)\right] < 0 \right\} \\
    &= 1 - \Pr \left\{\min_{\|\vdelta\|_{\infty} \le \epsilon} \left[ y \cdot \left( 1 \cdot \left(\cN(y, \sigma^2) + \delta_1 \right) + \sum_{i=2}^{d+1} \eta \cdot \left(\cN(y\eta, \sigma^2) + \delta_{i} \right) \right) \right] < 0 \right\} \\
    &\le 1 - \Pr \left\{y \cdot \left( 1 \cdot \left(\cN(y, \sigma^2) \right) + \sum_{i=2}^{d+1} \eta \cdot \left(\cN(y\eta, \sigma^2) - \epsilon \right) \right) < 0 \right\} \\
    &= 1 - \Pr \left\{\cN(1, \sigma^2) + \eta \sum_{i=2}^{d+1} \cN(\eta - \epsilon, \sigma^2) < 0 \right\} \\
    &= \Pr \left\{\cN(1, \sigma^2) + \eta \sum_{i=2}^{d+1} \cN(\eta - \epsilon, \sigma^2) > 0 \right\} \\
    &= \Pr \left\{\cN(0, 1) < \frac{1 - d\eta^2}{\sigma\sqrt{1+d\eta^2}} \right\}. \\
\end{aligned}
\end{equation}
Similarly, the adversarial accuracy of the robust classifier on the data $\cD$~(\ref{eq:mixGau}) is
\begin{equation}
\begin{aligned}
    \cA_{\textup{adv}}(f_{\textup{rob}}) 
    &= 1 - \underset{(\vx, y) \sim \cD}{\Pr} \left\{\exists \|\vdelta\|_{\infty} \le \epsilon, f_{\rob}(\vx+\vdelta) \neq y \right\} \\
    &= 1 - \underset{(\vx, y) \sim \cD}{\Pr} \left\{\min_{\|\vdelta\|_{\infty} \le \epsilon} [y \cdot f_{\rob}(\vx+\vdelta)] < 0 \right\} \\
    &= 1 - \Pr \left\{\min_{\|\vdelta\|_{\infty} \le \epsilon} \left[ y \cdot \left( 1 \cdot \left(\cN(y, \sigma^2) + \delta_1 \right) \right) \right] < 0 \right\} \\
    &= 1 - \Pr \left\{\min_{\|\vdelta\|_{\infty} \le \epsilon}  \left[\cN(1, \sigma^2) + \delta_1 \right]  < 0 \right\} \\
    &= 1 - \Pr \left\{  \cN(1, \sigma^2) - \epsilon  < 0 \right\} \\
    &= \Pr \left\{  \cN(1 - \epsilon, \sigma^2)  > 0 \right\} \\
    &= \Pr \left\{\cN(0, 1) < \frac{1 - 2\eta}{\sigma} \right\}. \\
\end{aligned}
\end{equation}
\end{proof}

\subsection{Proof of~\cref{thm:adv-harmless}}
\label{app:proof-thm-adv-harmless}

The following theorems rely on the analytical solution of optimal linear $\linf$-robust classifier on mixture Gaussian distributions. Concretely, the optimization problem is to minimize the adversarial risk on a distribution $\widehat{\cD}$ with a defense budget $\hat{\epsilon}$:
\begin{equation}
\label{equa.linear_adv_risk}
    \min_{f} \cR_{\adv}^{\hat{\epsilon}}(f, \widehat{\cD}), \quad \textup{where} \quad \cR_{\adv}^{\hat{\epsilon}}(f, \widehat{\cD}) \coloneqq \underset{(\boldsymbol{x}, y) \sim \widehat{\mathcal{D}}}{\mathbb{E}} \left[ \max_{\|\boldsymbol{\xi}\|_{\infty} \leq \hat{\epsilon}} \mathds{1} \left( \operatorname{sign}(\boldsymbol{w}^{\top} (\boldsymbol{x} + \boldsymbol{\xi}) + b) \neq y \right) \right],
\end{equation}
where $f(\vx) = \sign(\vw^{\top}\vx + b)$, and $\ind(\cdot)$ denotes the indicator function.

We note that optimal linear robust classifiers have been obtained for certain data distributions in previous work~\cite{tsipras2018robustness, ilyas2019adversarial, dobriban2020provable, javanmard2020precise, xu2021robust, tao2021provable}. Here, our goal is to establish similar optimal linear robust classifiers for the classification tasks in our setting. We only employ linear classifiers, since it is highly nontrivial to consider non-linearity for adversarial training on mixture Gaussian distributions~\cite{dobriban2020provable}.

\begin{lemma}
\label{lemma.b1}
Assume that the adversarial perturbation in data $\cT_{\adv}$~(\ref{eq:mixGau_adv}) is moderate such that $\eta/2 \le \epsilon < 1/2$. Then, minimizing the adversarial risk (\ref{equa.linear_adv_risk}) on the data $\cT_{\adv}$ with a defense budget $\epsilon$ can result in a classifier that assigns $0$ weight to the features $x_i$ for $i \geq 2$.
\end{lemma}

\begin{proof}
We prove the lemma by contradiction. 

The goal is to minimize the adversarial risk on the distribution $\cT_{\adv}$, which can be written as follows:

\begin{equation}
\label{eq:adfhaksdfghok}
\begin{aligned}
    \cR_{\adv}^{\epsilon}(f, \cT_{\adv})
    =& \underset{(\vx, y) \sim \cT_{\adv}}{\Pr} \left\{\exists \|\vdelta\|_{\infty} \le \epsilon, f(\vx+\vdelta) \neq y \right\} \\
    =& \underset{(\vx, y) \sim \cT_{\adv}}{\Pr} \left\{ \min_{\|\vdelta\|_{\infty} \le \epsilon} \left[y \cdot f(\vx+\vdelta)\right] < 0 \right\} \\
    =& \underset{(\vx, y) \sim \cT_{\adv}}{\Pr} \left\{ \max_{\|\vdelta\|_{\infty} \le \epsilon} \left[f(\vx+\vdelta)\right] > 0 \ | \ y = -1\right\} \cdot \underset{(\vx, y) \sim \cT_{\adv}}{\Pr} \left\{ y = -1 \right\} \\
    &+ \underset{(\vx, y) \sim \cT_{\adv}}{\Pr} \left\{ \min_{\|\vdelta\|_{\infty} \le \epsilon} \left[f(\vx+\vdelta)\right] < 0 \ | \ y = +1\right\} \cdot \underset{(\vx, y) \sim \cT_{\adv}}{\Pr} \left\{ y = +1 \right\} \\
    =& \underbrace{\Pr\left\{ \max_{\|\vdelta\|_{\infty} \le \epsilon} \left[ w_1(\cN(\epsilon-1, \sigma^2) + \delta_1) + \sum_{i=2}^{d+1} w_i (\cN(\epsilon-\eta, \sigma^2) + \delta_i) + b \right] > 0 \right\}}_{\cR_{\adv}^{\epsilon}(f, \cT_{\adv}^{(-1)})} \cdot \frac{1}{2} \\
    &+ \underbrace{\Pr\left\{ \min_{\|\vdelta\|_{\infty} \le \epsilon} \left[ w_1(\cN(1-\epsilon, \sigma^2) + \delta_1) + \sum_{i=2}^{d+1} w_i (\cN(\eta-\epsilon, \sigma^2) + \delta_i) + b \right] < 0 \right\}}_{\cR_{\adv}^{\epsilon}(f, \cT_{\adv}^{(+1)})} \cdot \frac{1}{2} \\
\end{aligned}
\end{equation}

Consider an optimal solution $\vw$ in which $w_i > 0$ for some $i\ge2$. Then, we have
\begin{equation}
\begin{aligned}
    \cR_{\adv}^{\epsilon}(f, \cT_{\adv}^{(-1)}) = \Pr\left\{ \underbrace{\sum_{j \neq i} \max_{\|\delta_j\| \le \epsilon} \left[ w_j (\cN(\epsilon - [\vw_{\nat}]_j, \sigma^2) + \delta_j) + b \right]}_{\bbA} + \underbrace{\max_{\|\delta_i\| \le \epsilon} \left[ w_i (\cN(\epsilon - \eta, \sigma^2) + \delta_i) \right]}_{\bbB} > 0 \right\}, \\
\end{aligned}
\end{equation}
where $\vw_{\nat} \coloneqq[1, \eta, \ldots, \eta]$ as in~\cref{eq:std_classifier}. Since $w_i > 0$, $\bbB$ is maximized when $\delta_i = \epsilon$. Thus, the contribution of terms depending on $w_i$ to $\bbB$ is a normally-distributed random variable with mean $2\epsilon - \eta$. Since $2\epsilon - \eta \ge 0$, setting $w_i$ to zero can only decrease the risk. This contradicts the optimality of $\vw$. Formally,
\begin{equation}
\begin{aligned}
    \cR_{\adv}^{\epsilon}(f, \cT_{\adv}^{(-1)}) = \Pr\left\{ \bbA + w_i \cN(2\epsilon-\eta, \sigma^2) > 0 \right\} > \Pr\left\{\bbA > 0\right\}.
\end{aligned}
\end{equation}
We can also assume $w_i<0$ and similar contradiction holds. Therefore, minimizing the adversarial risk on $\cT_{\adv}$ leads to $w_i=0$ for $i\ge 2$.
\end{proof}

\begin{lemma}
\label{lemma.b2}
Assume that the adversarial perturbation in data $\cT_{\adv}$~(\ref{eq:mixGau_adv}) is moderate such that $\eta/2 \le \epsilon < 1/2$. Then, minimizing the adversarial risk (\ref{equa.linear_adv_risk}) on the data $\cT_{\adv}$ with a defense budget $\epsilon$ results in a classifier that assigns a positive weight to the feature $x_1$.
\end{lemma}

\begin{proof}
We prove the lemma by contradiction. 

The goal is to minimize the adversarial risk on the distribution $\cT_{\adv}$, which has been written in~\cref{eq:adfhaksdfghok}.

Consider an optimal solution $\vw$ in which $w_1 \le 0$. Then, we have
\begin{equation}
\begin{aligned}
    \cR_{\adv}^{\epsilon}(f, \cT_{\adv}^{(-1)}) = \Pr\left\{ \underbrace{\sum_{j=2}^{d+1} \max_{\|\delta_j\| \le \epsilon} \left[ w_j (\cN(\epsilon - \eta, \sigma^2) + \delta_j) + b \right]}_{\bbC} + \underbrace{\max_{\|\delta_1\| \le \epsilon} \left[ w_1 (\cN(\epsilon - 1, \sigma^2) + \delta_1) \right]}_{\bbD} > 0 \right\}. \\
\end{aligned}
\end{equation}

Since $w_1 \le 0$, $\bbD$ is maximized when $\delta_1=-\epsilon$. Thus, the contribution of the term depending on $w_1$ to $\bbD$ is a normally-distributed random variable with mean $-1$. Since the mean is negative, setting $w_1$ to be positive can decrease the risk. This contradicts the optimality of $\vw$. Formally,
\begin{equation}
\begin{aligned}
    \cR_{\adv}^{\epsilon}(f, \cT_{\adv}^{(-1)}) = \Pr\left\{ \bbC + w_1 \cN(-\eta, \sigma^2) > 0 \right\} > \Pr\left\{ \bbC + p \cN(-\eta, \sigma^2) > 0 \right\}, 
\end{aligned}
\end{equation}
where $p > 0$ is any positive number. Therefore, minimizing the adversarial risk on $\cT_{\adv}$ leads to $w_1>0$.
\end{proof}

\textbf{Theorem~\ref{thm:adv-harmless} (restated).}
\emph{
Assume that the adversarial perturbation in the training data $\cT_{\adv}$~(\ref{eq:mixGau_adv}) is moderate such that $\eta/2 \le \epsilon < 1/2$. Then, the optimal linear $\linf$-robust classifier obtained by minimizing the adversarial risk on $\cT_{\adv}$ with a defense budget $\epsilon$ is equivalent to the robust classifier~(\ref{eq:rob_classifier}).
}

\begin{proof}
By \cref{lemma.b1} and \cref{lemma.b2}, we have $w_1>0$ and $w_i=0$ ($i\ge 2$) for an optimal linear $\linf$-robust classifier. Then, the adversarial risk on the distribution $\cT_{\adv}$ can be simplified by solving the inner maximization problem first. Formally,
\begin{equation}
\label{eq:asdgasdf}
\begin{aligned}
    \cR_{\adv}^{\epsilon}(f, \cT_{\adv})
    =& \underset{(\vx, y) \sim \cT_{\adv}}{\Pr} \left\{\exists \|\vdelta\|_{\infty} \le \epsilon, f(\vx+\vdelta) \neq y \right\} \\
    =& \underset{(\vx, y) \sim \cT_{\adv}}{\Pr} \left\{ \min_{\|\vdelta\|_{\infty} \le \epsilon} \left[y \cdot f(\vx+\vdelta)\right] < 0 \right\} \\
    =& \underset{(\vx, y) \sim \cT_{\adv}}{\Pr} \left\{ \max_{\|\vdelta\|_{\infty} \le \epsilon} \left[f(\vx+\vdelta)\right] > 0 \ | \ y = -1\right\} \cdot \underset{(\vx, y) \sim \cT_{\adv}}{\Pr} \left\{ y = -1 \right\} \\
    &+ \underset{(\vx, y) \sim \cT_{\adv}}{\Pr} \left\{ \min_{\|\vdelta\|_{\infty} \le \epsilon} \left[f(\vx+\vdelta)\right] < 0 \ | \ y = +1\right\} \cdot \underset{(\vx, y) \sim \cT_{\adv}}{\Pr} \left\{ y = +1 \right\} \\
    =& \Pr\left\{ \max_{\|\vdelta\|_{\infty} \le \epsilon} \left[ w_1(\cN(\epsilon-1, \sigma^2) + \delta_1) + b \right] > 0 \right\} \cdot \frac{1}{2} \\
    &+ \Pr\left\{ \min_{\|\vdelta\|_{\infty} \le \epsilon} \left[ w_1(\cN(1-\epsilon, \sigma^2) + \delta_1) + b \right] < 0 \right\} \cdot \frac{1}{2} \\
    =& \Pr\left\{ w_1\cN(2\epsilon-1, \sigma^2) + b > 0 \right\} \cdot \frac{1}{2} \\
    &+ \Pr\left\{ w_1\cN(1-2\epsilon, \sigma^2) + b < 0 \right\} \cdot \frac{1}{2}, \\
\end{aligned}
\end{equation}
which is equivalent to the natural risk on a mixture Gaussian distribution $\cD_{\text{tmp}}: \vx \sim \cN(y\cdot\vmu_{\text{tmp}}, \sigma^2\vI)$, where $\vmu_{\text{tmp}}=(1-2\epsilon, 0, \ldots, 0)$. We note that the Bayes optimal classifier for $\cD_{\text{tmp}}$ is $f_{\text{tmp}}(\vx) = \sign(\vmu_{\text{tmp}}^{\top}\vx)$. Specifically, the natural risk
\begin{equation}
\begin{aligned}
    \cR_{\adv}^{0}(f, \cD_{\text{tmp}})
    =& \underset{(\vx, y) \sim \cD_{\text{tmp}}}{\Pr} \left\{ f(\vx) \neq y \right\} \\
    =& \underset{(\vx, y) \sim \cD_{\text{tmp}}}{\Pr} \left\{ y \cdot f(\vx) < 0 \right\} \\
    =& \Pr\left\{w_1\cN(2\epsilon-1, \sigma^2)+b>0\right\} \cdot \frac{1}{2} \\
    &+ \Pr\left\{w_1\cN(1-2\epsilon, \sigma^2)+b<0\right\} \cdot \frac{1}{2}, \\
\end{aligned}
\end{equation}
which is minimized when $w_1 = 1 - 2\epsilon > 0$ and $b=0$. That is, minimizing the adversarial risk $\cR_{\adv}^{\epsilon}(f, \cT_{\adv})$ can lead to an optimal linear $\linf$-robust classifier $f_{\text{tmp}}(\vx)$.
Meanwhile, $f_{\text{tmp}}(\vx)$ is equivalent to the robust classifier~(\ref{eq:rob_classifier}). This concludes the proof of the theorem.
\end{proof}

\subsection{Proof of~\cref{thm:hyp-harmful}}
\label{app:proof-thm-hyp-harmful}

\begin{lemma}
\label{lemma.b3}
Minimizing the adversarial risk (\ref{equa.linear_adv_risk}) on the data $\cT_{\hyp}$~(\ref{eq:mixGau_hyp}) with a defense budget $\epsilon$ results in a classifier that assigns positive weights to the features $x_i$ for $i \geq 1$.
\end{lemma}

\begin{proof}
We prove the lemma by contradiction. 

The goal is to minimize the adversarial risk on the distribution $\cT_{\hyp}$, which can be written as follows:

\begin{equation}
\label{eq:hfgjhfgj}
\begin{aligned}
    \cR_{\adv}^{\epsilon}(f, \cT_{\hyp})
    =& \underset{(\vx, y) \sim \cT_{hyp}}{\Pr} \left\{\exists \|\vdelta\|_{\infty} \le \epsilon, f(\vx+\vdelta) \neq y \right\} \\
    =& \underset{(\vx, y) \sim \cT_{\hyp}}{\Pr} \left\{ \min_{\|\vdelta\|_{\infty} \le \epsilon} \left[y \cdot f(\vx+\vdelta)\right] < 0 \right\} \\
    =& \underset{(\vx, y) \sim \cT_{\hyp}}{\Pr} \left\{ \max_{\|\vdelta\|_{\infty} \le \epsilon} \left[f(\vx+\vdelta)\right] > 0 \ | \ y = -1\right\} \cdot \underset{(\vx, y) \sim \cT_{\hyp}}{\Pr} \left\{ y = -1 \right\} \\
    &+ \underset{(\vx, y) \sim \cT_{\hyp}}{\Pr} \left\{ \min_{\|\vdelta\|_{\infty} \le \epsilon} \left[f(\vx+\vdelta)\right] < 0 \ | \ y = +1\right\} \cdot \underset{(\vx, y) \sim \cT_{\hyp}}{\Pr} \left\{ y = +1 \right\} \\
    =& \underbrace{\Pr\left\{ \max_{\|\vdelta\|_{\infty} \le \epsilon} \left[ w_1(\cN(-1-\epsilon, \sigma^2) + \delta_1) + \sum_{i=2}^{d+1} w_i (\cN(-\eta-\epsilon, \sigma^2) + \delta_i) + b \right] > 0 \right\}}_{\cR_{\adv}^{\epsilon}(f, \cT_{\hyp}^{(-1)})} \cdot \frac{1}{2} \\
    &+ \underbrace{\Pr\left\{ \min_{\|\vdelta\|_{\infty} \le \epsilon} \left[ w_1(\cN(1+\epsilon, \sigma^2) + \delta_1) + \sum_{i=2}^{d+1} w_i (\cN(\eta+\epsilon, \sigma^2) + \delta_i) + b \right] < 0 \right\}}_{\cR_{\adv}^{\epsilon}(f, \cT_{\hyp}^{(+1)})} \cdot \frac{1}{2} \\
\end{aligned}
\end{equation}

Consider an optimal solution $\vw$ in which $w_i \le 0$ for some $i\ge1$. Then, we have
\begin{equation}
\begin{aligned}
    \cR_{\adv}^{\epsilon}(f, \cT_{\hyp}^{(-1)}) = \Pr\left\{ \underbrace{\sum_{j \neq i} \max_{\|\delta_j\| \le \epsilon} \left[ w_j (\cN( - [\vw_{\nat}]_j - \epsilon, \sigma^2) + \delta_j) + b \right]}_{\bbG} + \underbrace{\max_{\|\delta_i\| \le \epsilon} \left[ w_i (\cN(- [\vw_{\nat}]_i - \epsilon, \sigma^2) + \delta_i) \right]}_{\bbH} > 0 \right\}, \\
\end{aligned}
\end{equation}
where $\vw_{\nat} \coloneqq[1, \eta, \ldots, \eta]$ as in~\cref{eq:std_classifier}. Since $w_i \le 0$, $\bbH$ is maximized when $\delta_i = -\epsilon$. Thus, the contribution of terms depending on $w_i$ to $\bbH$ is a normally-distributed random variable with mean $-[\vw_{\nat}]_i-2\epsilon$. Since the mean is negative, setting $w_i$ to be positive can decrease the risk. This contradicts the optimality of $\vw$. Formally,
\begin{equation}
\begin{aligned}
    \cR_{\adv}^{\epsilon}(f, \cT_{\adv}^{(-1)}) = \Pr\left\{ \bbG + w_i \cN(-[\vw_{\nat}]_i-2\epsilon, \sigma^2) > 0 \right\} > \Pr\left\{ \bbG + p \cN(-[\vw_{\nat}]_i-2\epsilon, \sigma^2) > 0 \right\},
\end{aligned}
\end{equation}
where $p>0$ is any positive number. Therefore, minimizing the adversarial risk on $\cT_{\hyp}$ leads to $w_i>0$ for $i\ge1$.
\end{proof}

\textbf{Theorem~\ref{thm:hyp-harmful} (restated).}
\emph{
The optimal linear $\linf$-robust classifier obtained by minimizing the adversarial risk on the perturbed data $\cT_{\hyp}$~(\ref{eq:mixGau_hyp}) with a defense budget $\epsilon$ is equivalent to the natural classifier~(\ref{eq:std_classifier}).
}

\begin{proof}
By~\cref{lemma.b3}, we have $w_i>0$ for $i\ge 1$ for an optimal linear $\linf$-robust classifier. Then, we have
\begin{equation}
\label{eq:jertyetry}
\begin{aligned}
    \cR_{\adv}^{\epsilon}(f, \cT_{\hyp})
    =& \underset{(\vx, y) \sim \cT_{\hyp}}{\Pr} \left\{\exists \|\vdelta\|_{\infty} \le \epsilon, f(\vx+\vdelta) \neq y \right\} \\
    =& \underset{(\vx, y) \sim \cT_{\hyp}}{\Pr} \left\{ \min_{\|\vdelta\|_{\infty} \le \epsilon} \left[y \cdot f(\vx+\vdelta)\right] < 0 \right\} \\
    =& \underset{(\vx, y) \sim \cT_{\hyp}}{\Pr} \left\{ \max_{\|\vdelta\|_{\infty} \le \epsilon} \left[f(\vx+\vdelta)\right] > 0 \ | \ y = -1\right\} \cdot \underset{(\vx, y) \sim \cT_{\hyp}}{\Pr} \left\{ y = -1 \right\} \\
    &+ \underset{(\vx, y) \sim \cT_{\hyp}}{\Pr} \left\{ \min_{\|\vdelta\|_{\infty} \le \epsilon} \left[f(\vx+\vdelta)\right] < 0 \ | \ y = +1\right\} \cdot \underset{(\vx, y) \sim \cT_{\hyp}}{\Pr} \left\{ y = +1 \right\} \\
    =& \Pr\left\{ \max_{\|\delta_1\|_{\infty} \le \epsilon} \left[ w_1(\cN(-1-\epsilon, \sigma^2) + \delta_1) \right] + \sum_{i=2}^{d+1} \max_{\|\delta_i\|_{\infty} \le \epsilon} \left[ w_i(\cN(-\eta-\epsilon) + \delta_i) \right] + b > 0 \right\} \cdot \frac{1}{2} \\
    &+ \Pr\left\{ \min_{\|\delta_1\|_{\infty} \le \epsilon} \left[ w_1(\cN(1+\epsilon, \sigma^2) + \delta_1) \right] + \sum_{i=2}^{d+1} \min_{\|\delta_i\|_{\infty} \le \epsilon} \left[ w_i(\cN(\eta+\epsilon) + \delta_i) \right] + b < 0 \right\} \cdot \frac{1}{2} \\
    =& \Pr\left\{ w_1\cN(-1, \sigma^2) + \sum_{i=2}^{d+1} w_i\cN(-\eta, \sigma^2) + b > 0 \right\} \cdot \frac{1}{2} \\
    &+ \Pr\left\{ w_1\cN(1, \sigma^2) + \sum_{i=2}^{d+1} w_i\cN(\eta, \sigma^2) + b < 0 \right\} \cdot \frac{1}{2}, \\
\end{aligned}
\end{equation}
which is equivalent to the natural risk on the mixture Gaussian distribution $\cD: \vx \sim \cN(y\cdot\vw_{\nat}, \sigma^2\vI)$, where $\vw_{\nat}=(1, \eta, \ldots, \eta)$. We note that the Bayes optimal classifier for $\cD$ is $f_{\nat}(\vx) = \sign(\vw_{\nat}^{\top}\vx)$. Specifically, the natural risk
\begin{equation}
\begin{aligned}
    \cR_{\adv}^{0}(f, \cD)
    =& \underset{(\vx, y) \sim \cD}{\Pr} \left\{ f(\vx) \neq y \right\} \\
    =& \underset{(\vx, y) \sim \cD}{\Pr} \left\{ y \cdot f(\vx) < 0 \right\} \\
    =& \Pr\left\{w_1\cN(-1, \sigma^2) + \sum_{i=2}^{d+1} w_i \cN(-\eta, \sigma^2) + b > 0\right\} \cdot \frac{1}{2} \\
    &+ \Pr\left\{w_1\cN(1, \sigma^2) + \sum_{i=2}^{d+1} w_i \cN(\eta, \sigma^2) + b < 0\right\} \cdot \frac{1}{2}, \\
\end{aligned}
\end{equation}
which is minimized when $w_1 = 1$, $w_i=\eta$ for $i\ge 2$, and $b=0$. That is, minimizing the adversarial risk $\cR_{\adv}^{\epsilon}(f, \cT_{\hyp})$ can lead to an optimal linear $\linf$-robust classifier $f_{\nat}(\vx)$, which is equivalent to the natural classifier~(\ref{eq:std_classifier}). This concludes the proof of the theorem.

\end{proof}

\subsection{Proof of~\cref{thm:eps-eta-necessary}}
\label{app:proof-thm-eps-eta-necessary}

\begin{lemma}
\label{lemma.b4}
Minimizing the adversarial risk (\ref{equa.linear_adv_risk}) on the data $\cT_{\hyp}$~(\ref{eq:mixGau_hyp}) with a defense budget $\epsilon+\eta$ can result in a classifier that assigns $0$ weight to the features $x_i$ for $i \geq 2$.
\end{lemma}

\begin{proof}

The goal is to minimize the adversarial risk on the distribution $\cT_{\hyp}$, which can be written as follows:

\begin{equation}
\label{eq:jdfgjsdfg}
\begin{aligned}
    \cR_{\adv}^{\epsilon+\eta}(f, \cT_{\hyp})
    =& \underset{(\vx, y) \sim \cT_{\hyp}}{\Pr} \left\{\exists \|\vdelta\|_{\infty} \le \epsilon+\eta, f(\vx+\vdelta) \neq y \right\} \\
    =& \underset{(\vx, y) \sim \cT_{\hyp}}{\Pr} \left\{ \min_{\|\vdelta\|_{\infty} \le \epsilon+\eta} \left[y \cdot f(\vx+\vdelta)\right] < 0 \right\} \\
    =& \underset{(\vx, y) \sim \cT_{\hyp}}{\Pr} \left\{ \max_{\|\vdelta\|_{\infty} \le \epsilon+\eta} \left[f(\vx+\vdelta)\right] > 0 \ | \ y = -1\right\} \cdot \underset{(\vx, y) \sim \cT_{\hyp}}{\Pr} \left\{ y = -1 \right\} \\
    &+ \underset{(\vx, y) \sim \cT_{\hyp}}{\Pr} \left\{ \min_{\|\vdelta\|_{\infty} \le \epsilon+\eta} \left[f(\vx+\vdelta)\right] < 0 \ | \ y = +1\right\} \cdot \underset{(\vx, y) \sim \cT_{\hyp}}{\Pr} \left\{ y = +1 \right\} \\
    =& \underbrace{\Pr\left\{ \max_{\|\vdelta\|_{\infty} \le \epsilon+\eta} \left[ w_1(\cN(-1-\epsilon, \sigma^2) + \delta_1) + \sum_{i=2}^{d+1} w_i (\cN(-\eta-\epsilon, \sigma^2) + \delta_i) + b \right] > 0 \right\}}_{\cR_{\adv}^{\epsilon+\eta}(f, \cT_{\hyp}^{(-1)})} \cdot \frac{1}{2} \\
    &+ \underbrace{\Pr\left\{ \min_{\|\vdelta\|_{\infty} \le \epsilon+\eta} \left[ w_1(\cN(1+\epsilon, \sigma^2) + \delta_1) + \sum_{i=2}^{d+1} w_i (\cN(\eta+\epsilon, \sigma^2) + \delta_i) + b \right] < 0 \right\}}_{\cR_{\adv}^{\epsilon+\eta}(f, \cT_{\hyp}^{(+1)})} \cdot \frac{1}{2} \\
\end{aligned}
\end{equation}

Consider an optimal solution $\vw$ in which $w_i > 0$ for some $i\ge2$. Then, we have
\begin{equation}
\begin{aligned}
    \cR_{\adv}^{\epsilon+\eta}(f, \cT_{\hyp}^{(-1)}) = \Pr\left\{ \underbrace{\sum_{j \neq i} \max_{\|\delta_j\| \le \epsilon+\eta} \left[ w_j (\cN(-[\vw_{\nat}]_j-\epsilon, \sigma^2) + \delta_j) + b \right]}_{\bbI} + \underbrace{\max_{\|\delta_i\| \le \epsilon+\eta} \left[ w_i (\cN(-\eta-\epsilon, \sigma^2) + \delta_i) \right]}_{\bbJ} > 0 \right\}, \\
\end{aligned}
\end{equation}
where $\vw_{\nat} \coloneqq[1, \eta, \ldots, \eta]$. Since $w_i > 0$, $\bbJ$ is maximized when $\delta_i = \epsilon + \eta$. Thus, the contribution of terms depending on $w_i$ to $\bbJ$ is a normally-distributed random variable with mean $0$. Thus, setting $w_i$ to zero will not increase the risk. Formally, we have
\begin{equation}
\begin{aligned}
    \cR_{\adv}^{\epsilon+\eta}(f, \cT_{\hyp}^{(-1)}) = \Pr\left\{ \bbI + w_i \cN(0, \sigma^2) > 0 \right\} \ge \Pr\left\{\bbI > 0\right\}.
\end{aligned}
\end{equation}
We can also assume $w_i<0$ and a similar argument holds. Similar arguments also hold for $\cR_{\adv}^{\epsilon+\eta}(f, \cT_{\hyp}^{(+1)})$. Therefore, minimizing the adversarial risk on $\cT_{\hyp}$ can lead to $w_i=0$ for $i\ge 2$.
\end{proof}

\textbf{Theorem~\ref{thm:eps-eta-necessary} (restated).}
\emph{
The optimal linear $\linf$-robust classifier obtained by minimizing the adversarial risk on the perturbed data $\cT_{\hyp}$~(\ref{eq:mixGau_hyp}) with a defense budget $\epsilon+\eta$ is equivalent to the robust classifier~(\ref{eq:rob_classifier}). Moreover, any defense budget lower than $\epsilon+\eta$ will yield classifiers that still rely on all the non-robust features.
}

\begin{proof}
By \cref{lemma.b4}, we have $w_i=0$ ($i\ge 2$) for an optimal linear $\linf$-robust classifier. Also, the robust classifier will assign a positive weight to the first feature. This is similar to the case in~\cref{lemma.b2} and we omit the proof here. Then, we have
\begin{equation}
\label{eq:qegsdhbrhy}
\begin{aligned}
    \cR_{\adv}^{\epsilon+\eta}(f, \cT_{\hyp})
    =& \underset{(\vx, y) \sim \cT_{\hyp}}{\Pr} \left\{\exists \|\vdelta\|_{\infty} \le \epsilon+\eta, f(\vx+\vdelta) \neq y \right\} \\
    =& \underset{(\vx, y) \sim \cT_{\hyp}}{\Pr} \left\{ \min_{\|\vdelta\|_{\infty} \le \epsilon+\eta} \left[y \cdot f(\vx+\vdelta)\right] < 0 \right\} \\
    =& \underset{(\vx, y) \sim \cT_{\hyp}}{\Pr} \left\{ \max_{\|\vdelta\|_{\infty} \le \epsilon+\eta} \left[f(\vx+\vdelta)\right] > 0 \ | \ y = -1\right\} \cdot \underset{(\vx, y) \sim \cT_{\hyp}}{\Pr} \left\{ y = -1 \right\} \\
    &+ \underset{(\vx, y) \sim \cT_{\hyp}}{\Pr} \left\{ \min_{\|\vdelta\|_{\infty} \le \epsilon+\eta} \left[f(\vx+\vdelta)\right] < 0 \ | \ y = +1\right\} \cdot \underset{(\vx, y) \sim \cT_{\hyp}}{\Pr} \left\{ y = +1 \right\} \\
    =& \Pr\left\{ \max_{\|\vdelta\|_{\infty} \le \epsilon+\eta} \left[ w_1(\cN(-1-\epsilon, \sigma^2) + \delta_1) + b \right] > 0 \right\} \cdot \frac{1}{2} \\
    &+ \Pr\left\{ \min_{\|\vdelta\|_{\infty} \le \epsilon+\eta} \left[ w_1(\cN(1+\epsilon, \sigma^2) + \delta_1) + b \right] < 0 \right\} \cdot \frac{1}{2} \\
    =& \Pr\left\{ w_1\cN(-1-\eta, \sigma^2) + b > 0 \right\} \cdot \frac{1}{2} \\
    &+ \Pr\left\{ w_1\cN(1-\eta, \sigma^2) + b < 0 \right\} \cdot \frac{1}{2}, \\
\end{aligned}
\end{equation}
which is equivalent to the natural risk on a mixture Gaussian distribution $\cD_{\text{tmp}}: \vx \sim \cN(y\cdot\vmu_{\text{tmp}}, \sigma^2\vI)$, where $\vmu_{\text{tmp}}=(1-\eta, 0, \ldots, 0)$. We note that the Bayes optimal classifier for $\cD_{\text{tmp}}$ is $f_{\text{tmp}}(\vx) = \sign(\vmu_{\text{tmp}}^{\top}\vx)$. Specifically, the natural risk
\begin{equation}
\begin{aligned}
    \cR_{\adv}^{0}(f, \cD_{\text{tmp}})
    =& \underset{(\vx, y) \sim \cD_{\text{tmp}}}{\Pr} \left\{ f(\vx) \neq y \right\} \\
    =& \underset{(\vx, y) \sim \cD_{\text{tmp}}}{\Pr} \left\{ y \cdot f(\vx) < 0 \right\} \\
    =& \Pr\left\{w_1\cN(-1-\eta, \sigma^2)+b>0\right\} \cdot \frac{1}{2} \\
    &+ \Pr\left\{w_1\cN(1-\eta, \sigma^2)+b<0\right\} \cdot \frac{1}{2}, \\
\end{aligned}
\end{equation}
which is minimized when $w_1 = 1 - \eta > 0$ and $b=0$. That is, minimizing the adversarial risk $\cR_{\adv}^{\epsilon+\eta}(f, \cT_{\hyp})$ can lead to an optimal linear $\linf$-robust classifier $f_{\text{tmp}}(\vx)$, which is equivalent to the robust classifier~(\ref{eq:rob_classifier}).

Moreover, when the defense budget $\epsilon_d$ is less than $\epsilon + \eta$, the condition in \cref{lemma.b4} no longer holds. Instead, in this case, the robust classifier will assign positive weights to the features (i.e., $w_i>0$ for $i \ge 1$). This is similar to the case in \cref{lemma.b3}, and thus we omit the proof here. Consequently, this yields classifiers that still rely on all the non-robust features.

\end{proof}

\subsection{Proof of~\cref{thm:two-eps-sufficient}}
\label{app:proof-thm-two-eps-sufficient}

\textbf{Theorem~\ref{thm:two-eps-sufficient} (restated).}
\emph{
For any data distribution and any adversary with an attack budget $\epsilon$, training models to minimize the adversarial risk with a defense budget $2\epsilon$ on the perturbed data is sufficient to ensure $\epsilon$-robustness.
}

\begin{proof}
For clarity, we rewrite the adversarial risk in~(\ref{eq:adv_risk}) with a defense budget $\epsilon$ as follows:
\begin{equation}
\label{eq:adv_risk_rewrite}
    \cR_{\adv}^{\epsilon}(f, \cT) \coloneqq \underset{(\vx, y) \in \cT}{\sum} \left[\max_{\|\vdelta\| \le \epsilon} \cL(f(\vx+\vdelta), y)\right],
\end{equation}
where $\cT = \{(\vx_i, y_i)\}_{i=1}^n$ denotes the empirical training data.

Consider any adversary with an attack budget $\epsilon$, who can perturb $\vx$ to $\vx+\vp$ such that $\|\vp\| \le \epsilon$. Then, the learner will receive a perturbed version of training data $\cT' = \{(\vx_i+\vp_i, y_i)\}_{i=1}^n$. 

For any perturbed data point $(\vx_i+\vp_i, y_i)$, we have
\begin{equation}
\begin{aligned}
    \max_{\|\vdelta\| \le 2\epsilon} \cL(f(\vx_i+\vp_i + \vdelta), y_i) 
    &= \max_{\|\vdelta\| \le \epsilon, \|\vxi\| \le \epsilon} \cL(f(\vx_i+\vp_i + \vdelta + \vxi), y_i) \\
    &\ge \max_{\|\vdelta\| \le \epsilon} \cL(f(\vx_i+\vp_i + \vdelta -\vp_i), y_i) \\
    &= \max_{\|\vdelta\| \le \epsilon} \cL(f(\vx_i + \vdelta), y_i).
\end{aligned}
\end{equation}

By summarizing the training points, we have
\begin{equation}
\begin{aligned}
    \cR_{\adv}^{2\epsilon}(f, \cT') \ge \cR_{\adv}^{\epsilon}(f, \cT).
\end{aligned}
\end{equation}

That is, the adversarial risk with a defense budget $2\epsilon$ on the perturbed data is an upper bound of the adversarial risk with a defense budget $\epsilon$ on the original data. Therefore, a defense budget $2\epsilon$ is sufficient to ensure the learning of $\epsilon$-robustness.
\end{proof}

\clearpage

\section{Experimental Settings}
\label{sec:experimental-settings}

% \textbf{Adversary knowledge.}
% We mostly consider a worst-case evaluation of the security of learning algorithms, which serves as an upper bound on the robustness degradation that may be incurred by the system under attack. In this case, an adversary is assumed to have full knowledge of the model architecture and the underlying data distribution. We additionally consider a weaker adversary, who only has access to a proportion of training data and targets a different architecture.

\paragraph{Adversary capability.}
% \textbf{Adversary capability.}
We focus on the clean-label setting, where an adversary can only provide correctly labeled but misleading training data. In this setting, the main constraint is to craft perturbations as small as possible~\cite{feng2019learning}. Thus, we consider an $\linf$ adversary with an \textit{attack budget} $\epsilon_{a}=8/255$ by following~\citet{huang2021unlearnable, pmlr-v139-yuan21b, tao2021provable, fowl2021adversarial}. We note that this constraint is consistent with common research on test-time adversarial examples~\citep{athalye2018obfuscated}.

\paragraph{Crafting details.}
% \textbf{Crafting details.}
We conduct stability attacks by applying the hypocritical perturbation into the training set. Unless otherwise specified, we craft the perturbations by solving the error-minimizing objective~(\ref{eq:hypocritical_perturbation}) with 100 steps of PGD, where a step size of $0.8/255$ is used by following~\citet{fowl2021adversarial}. Our crafting model is adversarially trained with a \textit{crafting budget} $\epsilon_c=0.25\epsilon_a$ for 10 epochs before generating perturbations. That is, setting $\epsilon_c=2/255$ performs best, as shown in~\cref{fig:training-budget-for-crafting-model}. 

\paragraph{Training details.}
% \textbf{Training details.}
We evaluate the effectiveness of the hypocritical perturbation on benchmark datasets including CIFAR-10/100~\cite{Krizhevsky09learningmultiple}, SVHN~\cite{netzer2011reading}, and Tiny-ImageNet~\cite{le2015tiny}. Unless otherwise specified, we use ResNet-18~\cite{he2016deep} as the default architecture for both the crafting model and the learning model.
For adversarial training, we mainly follow the settings in previous studies~\cite{zhang2019theoretically, wang2019improving, rice2020overfitting}. By convention, the \textit{defense budget} is equal to the attack budget, i.e., $\epsilon_d=8/255$. The networks are trained for 100 epochs using SGD with momentum 0.9, weight decay $5 \times 10^{-4}$, and an initial learning rate of 0.1 that is divided by 10 at the 75-th and 90-th epoch. Early stopping is done with holding out 1000 examples from the training set. Simple data augmentations such as random crop and horizontal flip are applied. The inner maximization problem during adversarial training is solved by 10-steps PGD (PGD-10) with step size $2/255$.

\section{Feature-level Analysis on CIFAR-10}

In~\cref{sec:hyp-harmful}, we theoretically showed that the hypocritical perturbation can cause the poisoned model to rely more on non-robust features, thus the natural accuracy of the adversarially trained model is increased while the robust accuracy is decreased. 
In this part, we aim to provide empirical evidence on the role of non-robust features in the success of our poisoning method on a benchmark dataset. In particular, we will demonstrate that our hypocritical perturbation successfully makes the poisoned model learn more non-robust features.

To show this, by following Section 3.2 of~\citet{ilyas2019adversarial}, we construct a training set where the only features that are useful for classification are the non-robust features (that are extracted from the poisoned model). The standard accuracy of the classifier trained on the constructed dataset can reflect how many non-robust features are learned by the poisoned model (denoted as $f$). To accomplish this, we modify each input-label pair $(\vx, y)$ as follows. We select a target class $t$ uniformly at random among classes. Then, we add a small adversarial perturbation to $\vx$ as follows:
$$
\vx_{\adv} = \underset{\|\vx'-\vx\|\le\epsilon}{\arg\min} \ell(f(\vx'), t).
$$
The resulting input-label pairs $(\vx_{\adv}, t)$ make up the new training set. Since the resulting inputs $\vx_{\adv}$ are nearly indistinguishable from the originals $\vx$, the label $t$ assigned to the modified input is simply incorrect to a human observer. Therefore, only the non-robust features in the training set are predictive, while the non-robust features are extracted from the poisoned model.

We compare the model poisoned by our hypocritical perturbation with the baseline model trained on clean data. These two models correspond to the second row and last row in~\cref{tab:bench-attack}, respectively. Using these two models, we construct two datasets in the above-mentioned manner, respectively. Then, two new predictors are trained on the two constructed datasets, respectively, and both predictors are evaluated on clean data. Training parameters follow exactly those adopted by~\citet{ilyas2019adversarial}. Our numerical results are summarized in~\cref{tab:feature-level-analysis}.
  
\begin{table*}[!h]
\centering
\caption{The predictive ability of the non-robust features learned by the poisoned model.}
\label{tab:feature-level-analysis}
% \magic
\begin{center}
\begin{small}
\begin{tabular}{@{}cc@{}}
\toprule
Model for constructing the training set & Standard accuracy on the original test set (\%) \\ \midrule
The baseline model   & 27.46  \\
The poisoned model   & \textbf{56.77} \\ \bottomrule
\end{tabular}
\end{small}
\end{center}
% \vspace{-1.5em}
\magic
\end{table*}

As shown in~\cref{tab:feature-level-analysis}, the non-robust features learned by the poisoned model are much more predictive than the baseline. This indicates that the effect of our poisoning method on the non-robust features learned by the poisoned model is validated empirically.

\section{Broader Impact}
\label{sec:broader-impact}

The attack method in this work might be used by an agent in the real world to damage the robust availability of a machine-learning-based system. We discourage this malicious behavior by presenting the threat model of stability attacks to the community. We further propose an adaptive defense to mitigate this issue. The adaptive defense would help to build a more secure and robust machine learning system in the real world. At the same time, the adaptive defense introduces an additional time cost to search for an appropriate defense budget, which might have a negative impact on carbon emission reduction. Furthermore, society should not be overly optimistic about AI safety, since the current studies mostly focus on perturbations bounded by simple norms (e.g., $\linf$ norm in this paper). There might exist perturbations beyond the $\ell_p$ ball in the real world, and we are still far from complete model robustness.

\section{On the Trade-off between Accuracy and Robustness}

An interesting implication of this work is that the hypocritical perturbation exploits the trade-off between standard generalization and adversarial robustness, a phenomenon that has been widely observed in existing works on adversarial training~\cite{tsipras2018robustness, zhang2019theoretically, dobriban2020provable, mehrabi2021fundamental, su2018robustness, yang2020closer}.

Prior work mainly observed that adversarial training improves robust accuracy at the cost of natural accuracy \textit{when the training data is clean}. An explanation for the phenomenon is that there are non-robust features in the original dataset, which are predictive yet brittle~\cite{tsipras2018robustness, ilyas2019adversarial}. Unlike prior work, the trade-off in this work occurs \textit{when the training data is hypocritically perturbed}. Specifically, we make the following observations:
\begin{enumerate}
  \item When trained on the hypocritically perturbed data, conventional adversarial training produces models with lower robust accuracy but higher natural accuracy (e.g., see~\cref{tab:bench-attack},~\cref{tab:bench-attack-datasets}, and~\cref{tab:bench-attack-architectures}).
  \item When trained on the hypocritically perturbed data, adversarial training with adaptive budget can improve robust accuracy while reducing natural accuracy (e.g., see~\cref{tab:bench-defense}).
\end{enumerate}

These two observations align well with our theoretical analyses in~\cref{sec:how-to-manipulate} and~\cref{sec:necessity-large-budget}, respectively. Concretely, our analyses suggest that the hypocritical perturbation works by reinforcing the non-robust features in the original data, so that the models adversarially trained on the manipulated data still rely on the non-robust features. In this way, the natural accuracy of the models increases because the non-robust features are predictive, while the robust accuracy decreases because the non-robust features are brittle. Furthermore, the effectiveness of the adaptive defense lies in the fact that the reinforced non-robust features can be neutralized by enlarging the defense budget of adversarial training. Thus, the adaptive defense improves robustness at the cost of accuracy.

Meanwhile, we note that it would be unsatisfactory that test robustness is improved at the cost of standard generalization. Several improvements have been proposed to alleviate this issue in the case where the training data is clean, such as RST~\cite{raghunathan2020understanding}, FAT~\cite{zhang2020attacks}, and SCORE~\cite{pang2022robustness}. Incorporating these advances would be helpful in resisting stability attacks, and we leave this as future work.

Finally, we remark that the focus of stability attacks is to degrade test robustness. For this reason, we do not impose additional restrictions on their impact on natural accuracy.
Having that said, as a method of stability attacks, the hypocritical perturbation is observed to improve natural accuracy while reducing robust accuracy. We note that this makes stability attacks more insidious. For example, if a poisoned model exhibits higher natural accuracy, practitioners would be more easily enticed to deploy it in a real-world system. However, as its robust accuracy is actually undesirably low, the system is prone to losing its normal function when encountering test-time perturbations. In short, the negative impacts of stability attacks are serious, even with higher natural accuracy. Thus, it is imperative to design better defense methods to mitigate the threat of stability attacks.

\end{document}